\documentclass[10pt, letterpaper]{article}

\usepackage{graphicx}
\usepackage[linesnumbered,boxruled,vlined]{algorithm2e} 
\usepackage{amsmath}
\usepackage{amssymb}
\usepackage{comment}
\usepackage{scrextend}
\usepackage{mathrsfs}
\usepackage{amsthm}
\newtheorem{proposition}{Proposition}[section]

\newtheorem{example}{Example}

\DeclareMathOperator*{\argmax}{arg\,max}
\DeclareMathOperator*{\argmin}{arg\,min}
\DeclareMathAlphabet{\altmathcal}{OMS}{cmsy}{m}{n}

\providecommand{\keywords}[1]
{
  \small	
  \textbf{\textit{Keywords---}} #1
}

\usepackage[caption=false]{subfig}
\usepackage{hyperref}
\usepackage{comment}
\usepackage{float}
\title{Solving the Best Subset Selection Problem via Suboptimal Algorithms}
\author{Vikram Singh$^{1}$, Min Sun$^{2}$  \\
        \small $^{1}$University of Central Oklahoma, Edmond, OK \\
                 \small  vsingh2@uco.edu \\
        \small $^{2}$The University of Alabama, Tuscaloosa, AL \\
                 \small msun@ua.edu\\
}
\date{March 2025} % Comment this line to show today's date

\begin{document}

\maketitle

\begin{abstract}
Best subset selection in linear regression is well known to be nonconvex and computationally challenging to solve, as the number of possible subsets grows rapidly with increasing dimensionality of the problem. As a result, finding the global optimal solution via an exact optimization method for a problem with dimensions of 1000s may take an impractical amount of CPU time. This suggests the importance of finding suboptimal procedures that can provide good approximate solutions using much less computational effort than exact methods. In this work, we introduce a new procedure and compare it with other popular suboptimal algorithms to solve the best subset selection problem. Extensive computational experiments using synthetic and real data have been performed. The results provide insights into the performance of these methods in different data settings. The new procedure is observed to be a competitive suboptimal algorithm for solving the best subset selection problem for high-dimensional data.

%Insert your abstract here. Include keywords, PACS and mathematical subject classification numbers as needed.
\keywords{Linear regression, Best subset selection, Suboptimal algorithms, High-dimensional data}
%\subclass{90C59 \and 65K05}
\end{abstract}
%===============================================================================================
\section{Introduction} 
%===============================================================================================
\label{intro}
Consider a linear regression model \( y=X\beta+\varepsilon \),
where \(y \in \mathbb{R}^{n}\) is a response vector, \(X \in \mathbb{R}^{n \times p}\) is a design matrix, \(\beta \in \mathbb{R}^{p}\) is a coefficient vector, and \(\varepsilon \in \mathbb{R}^{n}\) is a noise vector. The columns of $X$ have been standardized to have zero mean and unit $l_{2}$-norm. 
The best subset selection (BSS) problem can be formulated as
\begin{equation}\label{bsschp5}
    \min_{\beta \in \mathbb{R}^p} \; \; \|y-X\beta\|_{2}^{2} \quad \text{subject to} \quad \|\beta\|_{0} \leq k,  
    \tag{BSS}
\end{equation}
where $\|\cdot\|_{0}$ is a pseudo-norm defined as $\|\beta\|_{0}=\sum_{i=1}^{p} 1(\beta_{i} \neq 0 )$, and $1(\cdot)$ is the indicator function
\begin{equation*}
    1(\beta_{i})=
    \begin{cases}
    1 & \text{if}\;\beta_{i}\neq 0\\
    0 & \text{if}\;\beta_{i} =0
    \end{cases}.
\end{equation*}
The $\|\cdot\|_{0}$ constraint is also called a cardinality-constraint (CC) in \eqref{bsschp5}, and it makes \eqref{bsschp5} NP-hard (see \cite{natarajan1995sparse}) and computationally challenging to solve in large dimensions. The objective function in \eqref{bsschp5} is the standard residual sum of squares (RSS). As the value of $p$ increases in the \eqref{bsschp5}, the number of feasible subsets grows exponentially, making it difficult for an exact method to find a globally optimal solution to the \eqref{bsschp5} problem. Thus, designing good suboptimal algorithms with accelerated convergence speeds is a common practical strategy.

Due to the high computational cost of solving feature selection problems with a large number of features to choose from, some suboptimal algorithms have been proposed in the pattern recognition literature (see \cite{mucciardi1971comparison}, \cite{chang:1973}, \cite{marill1963effectiveness}, \cite{whitney1971fss}, \cite{pudil1994floating}). 
The forward selection method proposed by Whitney \cite{whitney1971fss} is another suboptimal algorithm used to select a subset of features by sequentially selecting one feature at a time that minimizes a given criterion function until the desired number of features has been selected. Another popular algorithm is backward selection, introduced by \cite{marill1963effectiveness}, which starts by including all features in the model and sequentially removing one feature at a time until the desired number of features in the model is reached. Also, \cite{bertsimasEtal:2015} used a discrete extension of the first-order methods in convex optimization to find an approximate solution to the \eqref{bsschp5} problem. We can use these algorithms to address problems of much higher dimensions to find a good approximate solution to the \eqref{bsschp5} problem. There are many suboptimal algorithms in the literature that we can use for solving the \eqref{bsschp5} problem. In this work, we propose a new suboptimal algorithm and compare it with four other popular suboptimal algorithms to solve the \eqref{bsschp5} problem. Next, we introduce some notation.  

\paragraph{Notation} All vectors will be column vectors unless specified otherwise. For a vector $x$, the notation $x_{i}$ signifies the $i$th component, and $x_{1:i}$ represents a vector with first $i$ components of $x$. The notation $X \in \mathbb{R}^{n \times p}$ represents a real $n \times p$ matrix. %For such a matrix $X_{i}$ gives the $i$th row, and $X_{\cdot j}$ gives the $j$th column of $X$.
For an index set $\altmathcal{I}$, $\altmathcal{I}_{1:i}$ represents first $i$ indices of $\altmathcal{I}$. $X_{\altmathcal{I}}$ represents a submatrix consisting of columns of $X$ indexed by the set $\altmathcal{I}$. The cardinality of the set $\altmathcal{I}$ is given by the notation $|\altmathcal{I}|$. 
 Let $\altmathcal{A}=\{1,...,p \}$ be the set of indices of all the available predictors and $\altmathcal{I}$ be the set of indices of the $k$ predictors already selected in the model. The set $\altmathcal{A} \backslash \altmathcal{I}$ represents the complement of the set $\altmathcal{I}$ w.r.t. $\altmathcal{A}$. We define
\[ q(X_{\altmathcal{I}}) = \min_{ \beta \in \mathbb{R}^{|\altmathcal{I}|}} \; \|y-X_{\altmathcal{I}}\beta \|_{2}^{2}, \]
by including the predictors indexed by set $\altmathcal{I}$ in the linear model.
 For a set $\altmathcal{J} \subseteq \altmathcal{I}$, we define the gain in the RSS after dropping the predictor $X_{\altmathcal{J}}$ from the model w.r.t. set $\altmathcal{I}$ as
\( G(\altmathcal{J})=q(X_{\altmathcal{I}\backslash \altmathcal{J}}) -q(X_{\altmathcal{I}}). \)
For a set $\altmathcal{J} \subseteq \altmathcal{A}\backslash \altmathcal{I}$, we define the reduction in the RSS after adding the predictor $X_{\altmathcal{J}}$ to the model w.r.t. set $\altmathcal{I}$ as
\( R(\altmathcal{J})=q(X_{\altmathcal{I}})-q(X_{\altmathcal{I}\, \cup \, \altmathcal{J}}). \) 
%======================================================================================================
\section{Four existing suboptimal algorithms}
%======================================================================================================
In this section, we briefly describe four suboptimal algorithms to solve \eqref{bsschp5} problem. 
\subsection{Forward selection}  %=========================================================
The forward selection introduced by \cite{whitney1971fss} is a ``bottom-up" procedure that starts from the empty model and adds one feature at a time until the model contains the desired number of features. \algorithmcfname{ \ref{alg:fs_bss}} describes this procedure to solve the \eqref{bsschp5} problem. Starting from an empty model, we add one predictor at each iteration, which gives the maximum reduction in the RSS by including it in the current model. We stop the procedure when $k$ predictors are included in the model.
\begin{algorithm}
    \KwIn{$p, k, X, y$}
    \KwOut{$\altmathcal{I}^{\ast}$}
    Initialize $\altmathcal{I}=\emptyset$ and $stop=0$. \\
    \While{stop=0}{
     Find an index $i$ which maximizes the reduction w.r.t. the set $\altmathcal{I}$, i.e. 
     \[  i = \argmax_{i \in \altmathcal{A} \backslash \altmathcal{I}} \; R(i). \]   \\
     Update $\altmathcal{I}=\altmathcal{I}\cup i $.\\
     \If{$|\altmathcal{I}|=k$}{
      Set $\altmathcal{I}^{\ast}=\altmathcal{I}$, $stop=1$, and \Return. \\
     }
    }
    \caption{Forward Selection for \eqref{bsschp5}}
    \label{alg:fs_bss}
\end{algorithm}
\subsection{Sequential forward floating selection} %=======================================
Forward selection suffers from the so-called ``nesting effect", meaning that once a predictor is included in the model, there is no way to exclude it later. 
In \cite{pudil1994floating}, the authors introduced a sequential forward floating selection (SFFS) procedure to improve the forward selection method. In contrast to forward selection, SFFS incorporates a check to remove a previously selected predictor from the current model, provided that we can find a better model. For a given iteration of \algorithmcfname{ \ref{alg:sffs_bss}}, the index set $\altmathcal{I}$ has $i$ predictors. Step \ref{alg:sffs_inclusion} selects a new predictor $t$ as in the forward selection to obtain a new model with $i$+1 predictors. Step \ref{alg:sffs_CondExclusion} selects the worst predictor $r$ in the current model. If $r=t$, the new model is acceptable, and we can select the next predictor. If $r$ differs from $t$, we will drop the predictor $r$ to obtain a model with $i$ predictors again. We further try eliminating the previously selected features in the search for a better model.
\begin{algorithm}
    \KwIn{$p$, $k$, $X$, $y$}
    \KwOut{$\altmathcal{I}^{\ast}$}
    Set $\altmathcal{A}=\{1,...,p\}$ and find an initial index set $\altmathcal{I}$ with $|\altmathcal{I}|=2$ using \algorithmcfname{ \ref{alg:fs_bss}}. \\
    Initialize $i=2$.\\
    \While{$i<k$}{
    \textit{(Inclusion)} Find an index $t \in \altmathcal{A}\backslash\altmathcal{I}$ which maximizes the reduction w.r.t the set $\altmathcal{I}$, i.e.
    \[ t=\argmax_{j\in \altmathcal{A} \backslash \altmathcal{I}}\;R(j),\] and set $\altmathcal{I}=\altmathcal{I}\cup t$. \label{alg:sffs_inclusion}\\
    \textit{(Conditional Exclusion)} Find an index $r \in \altmathcal{I}$ which minimizes the gain w.r.t. the set $\altmathcal{I}$, i.e.
    \[ r=\argmin_{j\in \altmathcal{I}}\; G(j).\] \label{alg:sffs_CondExclusion} \\
    \eIf{$r=t$}{ Set $i=i+1$, go to step \ref{alg:sffs_inclusion}. \\
    }{ % if r/=t
      Set $\altmathcal{\hat{I}}=\altmathcal{I}\backslash r$. \\
      \If{$i=2$}{
       Set $\altmathcal{I}=\altmathcal{\hat{I}}$ and go to step \ref{alg:sffs_inclusion}. \\
      }
    }
    \textit{(Continuation of Conditional Exclusion)}
    Find an index $s \in \altmathcal{\hat{I}}$ 
    \[ s=\argmin_{j\in\altmathcal{\hat{I}}}\;G(j).\] \label{alg:sffs_contCondExclusion} \\
    \eIf{$q(X_{\hat{\altmathcal{I}}\backslash s})\geq q(X_{\altmathcal{I}_{1\,:\,(i-1)})}$}{
     Set $\altmathcal{I}=\altmathcal{\hat{I}}$ and go to step \ref{alg:sffs_inclusion}. \\
    }{
     Set $\altmathcal{\hat{I}}=\altmathcal{\hat{I}}\backslash s$ and $i=i-1$.\\
     \eIf{$i=2$}{
      Set $\altmathcal{I}=\altmathcal{\hat{I}}$ and go to step \ref{alg:sffs_inclusion}.\\
     }{
      go to step \ref{alg:sffs_contCondExclusion}.\\
     }
    }
    }
    Set $\altmathcal{I}^{\ast}=\altmathcal{I}$ and \Return. \\
    \caption{Sequential Forward Floating Selection for \eqref{bsschp5}}
    \label{alg:sffs_bss}
\end{algorithm}
\subsection{Discrete first order method} %==================================================
\label{sec:dfo}
The discrete first-order (DFO) method, as presented in \cite{bertsimasEtal:2015}, is an extension of the first-order methods in convex optimization to solve the \eqref{bsschp5} problem. Consider the minimization problem
\begin{equation}
    \min_{\beta}\; g(\beta) \quad \text{subject to} \quad \|\beta\|_{0} \leq k,
\end{equation}
where $g(\beta)\geq 0$ is convex and has Lipschitz continuous gradient:
\begin{equation}\label{eqn:dfo_lipschitz}
    \| \nabla g(\beta) - \nabla g(\Tilde{\beta})\|\leq l \|\beta - \Tilde{\beta}\|.
\end{equation}
The following proposition gives a closed-form solution for the case when $g(\beta)=\|\beta - c \|_{2}^{2}$ for a given $c$. 
\begin{proposition}[Proposition 3, \cite{bertsimasEtal:2015}]\label{prop:dfo_prop1}
If $\Hat{\beta}$ is an optimal solution to the following problem,
\begin{equation}\label{eqn:dfo_prop1}
 \Hat{\beta} \in \argmin_{\|\beta\|_{0}\leq k} \; \|\beta - c \|_{2}^{2},  
\end{equation}
then it can be computed as follows: $\Hat{\beta}$ retains the $k$ largest (in absolute value) elements of $c \in \mathbb{R}^{p}$ and sets the rest to zero, i.e. if $|c_{(1)}| \geq |c_{(2)}| \geq . . . \geq |c_{(p)}|$, denote the ordered values of the absolute value of the vector $c$, then
\[ \Hat{\beta}_{i}=\begin{cases}
    c_{i}, & if \; i \; \in \{(1),(2),...,(k) \}, \\
     0  ,  & otherwise, \\    
\end{cases} \]
where $\Hat{\beta}_{i}$ is the $ith$ coordinate of $\Hat{\beta}$. Let $H_{k}(c)$ denote the set of solutions for \eqref{eqn:dfo_prop1}.  
\end{proposition}
The next proposition provides an upper bound for the function $g(\eta)$ around $g(\beta)$.
\begin{proposition}[Proposition 4, \cite{bertsimasEtal:2015}]\label{prop:dfo_prop2}
 For a convex function $g(\beta)$ satisfying condition \eqref{eqn:dfo_lipschitz} and for any $L\geq l$, we have 
 \begin{equation}
     g(\eta)\leq Q_{L}(\eta,\beta):= g(\beta) + \frac{L}{2}\|\eta - \beta \|_{2}^{2}+ (\eta - \beta)^{t}\nabla g(\beta)
 \end{equation}
 for all $\beta$, $\eta$ with equality holding at $\beta=\eta$.
\end{proposition}
Applying Proposition \ref{prop:dfo_prop1} to $Q_{L}(\eta, \beta)$ in Proposition \ref{prop:dfo_prop2} the authors deduced the following result
\[ \argmin_{\|\eta\|_{0}\leq k} \; Q_{L}(\eta,\beta) = H_{k}\left( \beta - \frac{1}{L} \nabla g(\beta) \right)
\]
where $H_{k}(\cdot)$ is as defined after proposition \ref{prop:dfo_prop1}. Algorithm \ref{alg:dfo} can be used to solve the \eqref{bsschp5} with $g(\beta)=\|y-X\beta\|_{2}^{2}$, $\nabla g(\beta)=-2X^{t}(y-X\beta)$, and $l=\lambda_{\max}(X^{t}X)$. For further details and the convergence analysis for \algorithmcfname{ \ref{alg:dfo}} see section 3 in \cite{bertsimasEtal:2015}.
\begin{algorithm}
    \KwIn{$p$, $k$, $g$, $L$}
    \KwOut{$\beta^{*},g^{*}$}
    Initialize $i=0$ and $stop=0$.\\
    Find an initial $\beta_{i}$ such that $\|\beta_{i}\|_{0}\leq k$.\\
    \While{stop=0}{
    Set $i=i+1$. \\
    Find $\beta_{i} \in H_{k}\left( \beta_{i-1} - \frac{1}{L}\nabla g(\beta_{i-1})  \right)$. \\
    \If{$g(\beta_{i})-g(\beta_{i-1})<\varepsilon$}{
     Set $\beta^{*}=\beta_{i}$, $g^{*}=g(\beta_{i})$, $stop=1$, and \Return.
    }
    }
    \caption{Discrete First Order method for \eqref{bsschp5}}
    \label{alg:dfo}
\end{algorithm}
\subsection{Genetic algorithm} %=====================================================
\label{sec:ga}
The genetic algorithm (GA) follows the Darwinian theory of evolution and is a popular method for finding approximate solutions to various optimization problems. The GA can be used to solve both continuous and discrete optimization problems (see \cite{goldberg1989genetic}). There are three major components in any GA.
\begin{enumerate}
    \item A population of chromosomes is maintained throughout the procedure.
    In particular, an initial pool of chromosomes is needed to start the procedure. It is usually selected randomly. It may also be obtained by other reasonable methods.
    \item Generate new chromosomes via crossover and/or mutation applied to the selected parents from the current population.
    \item Discard the worst chromosomes from the population using the fitness criteria.
\end{enumerate}
\begin{comment}
\begin{enumerate}
    \item \texit{(Population)} A population of chromosomes is maintained throughout the procedure.
    In particular, an initial pool of chromosomes is needed to start the procedure. It is usually selected randomly. It may also be obtained by other reasonable methods.
    \item \textit{(Reproduction)} Generate new chromosomes via crossover and/or mutation applied to the selected parents from the current population.
    \item \textit{(Survival of the Fittest)} Discard the worst chromosomes from the population using the fitness criteria.
\end{enumerate}
\end{comment}
Algorithm \ref{alg:genetic_bss} describes our version of GA in binary variables, where $s$ is the population size. $CR$ and $MR$ are the crossover and mutation rates, respectively. The domain of \eqref{bsschp5} is the set $\mathbb{R}^{p}$. A binary $p$-vector $y$ with exactly $p-k$ zeros is regarded as a candidate solution (a chromosome) for the \eqref{bsschp5} problem because there is a one-to-one correspondence between $y$ and the feasible point $z$ for the \eqref{bsschp5}. In particular, for a given $y$, we can find $z$ such that $z_{\altmathcal{A}\backslash\altmathcal{I}}=0$ and
$f(z)=\min \; q(X_{\altmathcal{I}})$ where $\altmathcal{I}=\{i : y_{i}=1, i=1,...,p \}$.
\paragraph{Initial population} Create an initial population $P$ (step \ref{alg:genetic_initial_pop}), which is a matrix of the size $p\times s$ with every column having exactly $p-k$ zeros randomly placed. Every column of $P$ represents a chromosome.
\paragraph{Reproduction} Construct two new matrices $U$ and $V$ by performing \textbf{crossover} of $CR$ percentage of chromosomes from the current population as follows: Select two distinct columns $u$ and $v$ at random from $P$, and find a random crossover site given by the index $i$. Set two new arrays $\hat{u}$ and $\hat{v}$ such that
\[ \hat{u}=\begin{bmatrix}
    u_{1:i} \\
    v_{i+1:p}
\end{bmatrix} \text{ and } \hat{v}=\begin{bmatrix}
    v_{1:i}\\
    u_{i+1:p}
\end{bmatrix}.  \] 
Adjust $\hat{u}$ and $\hat{v}$ to have exactly $p-k$ zeros. To introduce diversity in the population, we perform \textbf{mutation} on $U$ and $V$ as follows: Randomly select $MR$ percentage of indices from $U$ and $V$, and assign the opposite values. Adjust indices to have exactly $k$ number of ones in each column. Add $U$ and $V$ to the current population $P$ and expand the array $F$ by evaluating the fitness criteria for the new offsprings.
\paragraph{Survival of the fittest} Discard the columns of $P$ that correspond to bigger cost function values to keep the size of the population to $s$. Adjust $F$ by deleting values for the discarded columns of $P$. 

The scalar $rand$ in step \ref{alg:genetic_rand} of \algorithmcfname{ \ref{alg:genetic_bss}} is sampled from the uniform distribution $U(0,1)$ at each call. The algorithm will stop once all the chromosomes in the current population become the same. As the final step, we find the corresponding feasible point $\beta^{*}$ to the \eqref{bsschp5}, which will be the final solution for GA.
\begin{algorithm}
    \KwIn{$s, p, q, k, CR, MR$}
    \KwOut{$\beta^{*}$}
    Initialize $stop=0$.\\
    \While{stop=0}{
      Construct an initial population $P$ of $s$ chromosomes.  Construct a corresponding array $F$ of criteria function 
      values for each chromosome. \label{alg:genetic_initial_pop}\\
      Construct new offspring $U$ and $V$ as follows. Initialize $r=0$. \label{alg:genetic_new_offsprings_step} \\
      \For{$t=1$ \KwTo $s$}{
        \If{$rand<CR$}{ \label{alg:genetic_rand}
        Set $r=r+1$.\\
        Generate two new chromosomes $\hat{u}$ and $\hat{v}$. \\ 
        Set $U_{\cdot r}=\hat{u}$ and $V_{\cdot r}=\hat{v}$.\\
        }
      }
      Mutate $MR$ percentage of bits in $U$ and $V$. \\
      Add $U$ and $V$ to the current population $P$ and update $F$ accordingly. \\
      Discard the chromosomes with bigger criteria function values to keep the population size at $s$. Adjust $P$ and $F$ accordingly. \\
      \If{$P_{i}=P_{j}$ $\forall$ $i,j$}{
      Set $stop=1$, and \Return \\
      }
    }
    \caption{Genetic Algorithm for \eqref{bsschp5}}
    \label{alg:genetic_bss}
\end{algorithm}
%========================================================================================================
\section{New suboptimal algorithm}
%========================================================================================================
%\subsection{Sequential feature swapping} %=======================================================
\label{sec:chp5_sfs}
We introduce the sequential feature swapping (SFS) given by \algorithmcfname{ \ref{alg:sfs_bss}}, which is an iterative procedure based on the idea of swapping bad predictors from the currently selected model with good predictors not in the model at each iteration, starting from a model with $k$ predictors.  
This algorithm is attractive because it usually provides a good approximate solution to the \eqref{bsschp5} problem with much less computational effort.  
 \begin{algorithm}
 \KwIn{$p$, $k$, $t$, $X$, $y$}
 \KwOut{$\altmathcal{I}^{\ast}$}
      Set $\altmathcal{A}=\{1,...,p\}$ and find an initial index set $\altmathcal{I}$ of $k$ predictors.\\ 
     Set $k=0$ and $stop=0$.\\
     \While{stop=0}{
       \textit{(Drop)} Find an index set $\altmathcal{J} \subseteq \altmathcal{I}$ with $|\altmathcal{J}|=t$ which minimizes the gain w.r.t the set $\altmathcal{I}$, i.e. 
       \[  \altmathcal{J}=\argmin_{\altmathcal{S} \subseteq \altmathcal{I}} \; \; G(\altmathcal{S}),   
       \] and set $\hat{\altmathcal{I}}=\altmathcal{I} \backslash \altmathcal{J}$.      \\
       \textit{(Pick)} Find the index set $\altmathcal{Q} \subseteq \altmathcal{A}\backslash \altmathcal{I}$ with $|\altmathcal{Q}|=t$ which maximizes the reduction w.r.t. the set $\hat{\altmathcal{I}}$, i.e. 
       \[  \altmathcal{Q}= \argmax_{\altmathcal{S} \subseteq \altmathcal{A} \backslash \altmathcal{I}} \; \; R(\altmathcal{S}). \]   \\
       \eIf{$ q(X_{ \hat{\altmathcal{I}}\cup \altmathcal{Q}} ) < q(X_{\altmathcal{I}}) $}{
       \textit{(Switch)}
        Update $\altmathcal{I}= \hat{\altmathcal{I}} \cup \altmathcal{Q}$. \\
       }{Set $\altmathcal{I}^{\ast}=\altmathcal{I}$, $stop=1$, and \Return.
       }
     }
     \caption{Sequential Feature Swapping for \eqref{bsschp5} }
     \label{alg:sfs_bss}
\end{algorithm}
\begin{proposition}
Algorithm \ref{alg:sfs_bss} terminates after a finite number of iterations.
\end{proposition}
\begin{proof}
At each iteration, for the updated set $\altmathcal{I}$, the value of $q(X_{\altmathcal{I}})$ decreases monotonically, and $q(X_{\altmathcal{I}})$ is bounded from below by $q(X_{\altmathcal{I}^{*}})$, where $\altmathcal{I}^{*}$ is the optimal set of indices of $k$ predictors for the \eqref{bsschp5} problem. Hence, the algorithm \ref{alg:sfs_bss} terminates after a finite number of iterations.
\end{proof} 
In contrast to a forward (backward) selection algorithm where once a predictor is selected (removed from the model), it cannot be removed (added back again) from the model, Algorithm \ref{alg:sfs_bss} can remove some previously selected predictor or add a previously removed predictor to obtain a new model with a smaller RSS. Thus, SFS offers more power of search while still retaining a sense of monotonicity. When parameter t=1, we change the selected model by one predictor at a time. When t=2, we add and drop two predictors at a time. Although choosing t=2 always improves the SFS locally, there is no guarantee that the final model selected by choosing t=2 is always better than the model selected by choosing t=1.

We note that Algorithm \ref{alg:sfs_bss} shares a similar idea in Efroymson's stepwise algorithm (see \cite{miller2002subset}) with the difference that Algorithm \ref{alg:sfs_bss} starts when we already have $k$ predictors selected in the model, whereas the Efroymson stepwise algorithm starts with no predictor in the model and sequentially selects a new predictor with a check included to drop some previously selected predictor from the model at each step. We choose the initial set $\altmathcal{I}$ containing indices of the $k$ predictors with the largest magnitude in the OLS solution. Additionally, see the ``splicing algorithm" described in \cite{zhu2020polynomial}, which further generalizes the idea of Algorithm \ref{alg:sfs_bss} by letting the parameter t float. 
%==================================================================================================
\section{Numerical results}
%==================================================================================================
\label{sec:results}
We provide numerical results comparing the following algorithms to solve the \eqref{bsschp5} problem.

\begin{labeling}{DFOn }  % DFOn is the longest label
    \item[\textbf{SFS1}] Sequential Feature Swapping given by \algorithmcfname{ \ref{alg:sfs_bss}} with t=1;
    \item[\textbf{SFS2}] Sequential Feature Swapping given by \algorithmcfname{ \ref{alg:sfs_bss}} with t=2;
    \item[\textbf{FS}] Forward Selection, given by \algorithmcfname{ \ref{alg:fs_bss}};
    \item[\textbf{SFFS}] Sequential Forward Floating Selection given by \algorithmcfname{ \ref{alg:sffs_bss}};
    \item[\textbf{GA}]   Genetic Algorithm given by \algorithmcfname{ \ref{alg:genetic_bss}};
    \item[\textbf{DFO}]  Discrete First Order method given by Algorithm \ref{alg:dfo};
    \item[\textbf{DFOn}] Discrete First Order method given by Algorithm \ref{alg:dfo} with $n$ runs using different starting points and taking the best result as the final output.
\end{labeling}
We tested the performance of the algorithms on both real and synthetic data sets. We consider both the overdetermined (OD, $p<n$) and the underdetermined (UD, $p\gg n$) case.
\paragraph{Computer specifications} All the algorithms have been implemented in \textsc{MATLAB} %testing has been done on the Windows 64-bit operating system with Intel(R) Core(TM) i7-8565U CPU $@$ 1.80GHz, with 16 GB of RAM. We restricted all the algorithms to run on a single core for fair comparison.
 and testing has been done on The University of Alabama High-Performance Computer using \textsc{MATLAB} R2021a. To have a fair comparison, each algorithm is restricted to run on exactly one core on the Linux 64-bit operating system with Intel(R) Xeon(R) Gold 6126 CPU @ 2.60GHz. 
\subsection{Performance profiles and box plots}
Our numerical experiment on a large number of examples generated a significant amount of output. The output results are summarized mainly in visual approaches. To visualize the output data conveniently we adopt two commonly used types of plots: performance profiles and box plots. For any number of examples tested and any predetermined performance measure, each solver could be associated with any such plots. All solvers can then be visually compared by grouping these plots together.

Performance profiles as introduced in \cite{dolan2002perprof} use the idea of comparing the ratio of one solver’s performance measure with the best (minimum) performance measure among all solvers for that problem. More specifically, for each problem $p$ and solver $s$, let $t_{p,s}$ define a performance measure that we want to compare. We calculate the performance ratio as 
\[ r_{p,s}=\frac{t_{p,s}}{\min \{t_{p,s}:s\in S \} }\]
where $S$ is the set of solvers. We plot the resulting data using an empirical CDF plot. 

Box plots for output data give us a visualization of five summary statistics of any given performance measure. The bottom and top of each box are the 25th and 75th percentiles of the sample. The horizontal line in the middle of the box shows the median of the sample data. The vertical lines extending from the box go to the upper and lower extreme. The observations beyond the upper and lower extremes are marked as outliers. For all the box plots in this article, a value is considered as an outlier (marked as $+$ sign) if it is more than 1.5 times the interquartile range away from the bottom or top of the box.  
\subsection{Relative gap percentage and CPU time as a performance measure} %============================
To compare the solution quality we use the Relative Gap $\%$ defined as 
\[ \left(  \frac{\Tilde{f}- f^{*}}{f^{*}} \right)100,\] 
where, for a particular example, $f^{*}$ is the overall best function value found by any algorithm for that example, and $\Tilde{f}$ is the function value from a given algorithm. A small Relative Gap $\%$ implies that the solution from a given algorithm is close to the best solution. 

The second performance measure we compare is the CPU time taken by an algorithm to find the optimal solution. A smaller CPU time with a compatible final function value $\Tilde{f}$ is a good indication of the efficiency of an algorithm.
\subsection{Synthetic data sets} %======================================================================
In this section, we present the test results for experiments using synthetic data sets. We intentionally diversified dimensions, structure of underlying true solutions, degree of noise correlations, and signal to noise ratios. 
\subsubsection{Test data setup} %=======================================================================
We have tested the suboptimal methods using synthetic data sets constructed as follows. Firstly, we find the design matrix $X \in \mathbb{R}^{n \times p}$ by sampling each row (i.i.d.) from a $p$-dimensional multivariate normal distribution $N(0,\Sigma)$, with mean zero and covariance matrix $\Sigma$. We normalize the columns of $X$ to have zero mean and unit $l_{2}$-norm. We construct the coefficient vector $\beta^{0}$. We choose a noise vector $\varepsilon$ (i.i.d.) from the normal distribution $N(0,\sigma^{2})$, where the variance $\sigma^{2}$ is chosen according to the given signal-to-noise (SNR) ratio defined as
\( \text{SNR}:=\frac{\|X\beta^{0} \|_{2}^{2}}{ \sigma^{2} }.  \)
Finally, we get the response vector $y$ using $y=X\beta^{0}+\varepsilon$. 
\tablename{ \ref{tab:sma_med_lar_eg_setup}} shows test examples we used in 4 different dimension groups for both OD and UD cases. We chose $k_{0}$ (the number of non-zero entries in $\beta^{0}$) to be 10. We want to select the best set of 5 and 10 predictors, i.e. $k\in \{5,10\}$. We tested three examples by varying $\beta^{0}$.   
\begin{table}[h!]
    \centering
    \caption{Test examples setup}
    \label{tab:sma_med_lar_eg_setup}
    \begin{tabular}{llll}
    \hline\noalign{\smallskip}
     Type & $p$ & $n$ for OD & $n$ for UD   \\
    \noalign{\smallskip}\hline\noalign{\smallskip} 
     small-1 & 20 & 100 & 10 \\
     small-2 & 40 & 200 & 20 \\
     small-3 & 60 & 300 & 30 \\
     small-4 & 80 & 400 & 40 \\
     \noalign{\smallskip}\hline\noalign{\smallskip}
     medium-1 & 200 & 1000 & 100 \\
     medium-2 & 300 & 1000 & 100 \\
     medium-3 & 400 & 2000 & 100 \\
     medium-4 & 500 & 2000 & 100 \\
     \noalign{\smallskip}\hline\noalign{\smallskip}
     large-1 & 800 & 4000 & 200 \\
     large-2 & 1000 & 4000 & 200 \\
     large-3 & 1500 & 8000 & 300 \\
     large-4 & 2000 & 8000 & 300 \\
     \noalign{\smallskip}\hline
     %ultra large-1 & 3000 & 10000 & 500 \\
     %ultra large-2 & 5000 & 15000 & 500 \\
     %\hline
    \end{tabular}  
\end{table}
% this is typeOne in the testing
\begin{example}
 Generate $\beta^{0}$ such that $\beta_{i}^{0}=1$ for $k_{0}$ equally spaced indices from the set $\{1,...,p \}$, rounding to the greatest integer if needed.   
\end{example}

% typeTwo in the testing
\begin{example}
Generate $\beta^{0}$ by assigning the first $k_{0}$ entries as $1$, that is $\beta_{i}^{0}=1$ for $i=1,...,k_{0}$. 
\end{example}

% type random in the testing
\begin{example}
Generate $\beta^{0}$ by picking a random subset from $\{1,...,p\}$ of $k_{0}$ indices, and then we assign random integer values between $1$ and $5$ to those indexed variables.   
\end{example}
We run the three examples in small, medium, and large dimension settings as given in \tablename{ \ref{tab:sma_med_lar_eg_setup}} with 4 SNR values, SNR $\in$ $ \{0.05, 0.5, 1, 5 \}$. 
 We used a CPU time limit of 600 seconds for all the algorithms. For GA and DFO, we also used a max iteration limit of 10,00,000 for all the runs and stopped the algorithm once this limit had been reached, returning the current best solution as the final output. We compare two aspects of these algorithms, solution quality and CPU time. In the testing of GA, the constants $CR$ and $MR$ are set to 0.8 and 0.2, respectively, and the population size $s$ is set to 10. For the DFO method, the value of $L$ is taken to be $\lambda_{\max}(X^{t}X)$. For DFOn, we run DFO with 20 different initial points, each with a different support, and choose the best solution among all the runs. We compared the algorithms on both constant and exponentially correlated data.

\subsubsection{Test results with constant correlation data} %===========================================
For the synthetic data sets with constant correlation, the covariance matrix $\Sigma$ is chosen such that $\Sigma_{i,j}=0.8$ when $i\neq j$ and $\Sigma_{i,i}=1$.
\paragraph{Solution quality}
Figure \ref{fig:boxplotsOdEgsRelGapSfsFsSffsGaDfowrtSNR} shows box plots of the Relative Gap $\%$ for small, medium, and large dimension examples with data generated using constant correlation with four different SNR values for OD case. GA and DFOn are a lot worse than the other four algorithms. DFOn performs better than GA in every instance by providing a smaller median value than GA. Among the remaining four algorithms, SFFS performs better in most of the instances, and not worse in a few instances, than SFS1, SFS2, and FS, consistently providing the smallest Relative Gap $\%$ as can be seen in Figure \ref{fig:boxplotsOdEgsRelGapSfsFsSffsGaDfowrtSNR}, followed by FS. There is no clear winner between SFS1 and SFS2. 
% The results below are generated from the output 
%C:\Users\vikra\Desktop\testLLS\outLLSpackage05July\Ubxtype12R_6algSubOpt\ for
% box plots for Relative Gap % for 99b3e9c   and \outLLSpackage10Oct24\Ubx9d3fCcV9\  for 9d3f
 \begin{figure}[h!]
    %\centering
    \subfloat{\includegraphics[width=0.32\textwidth]{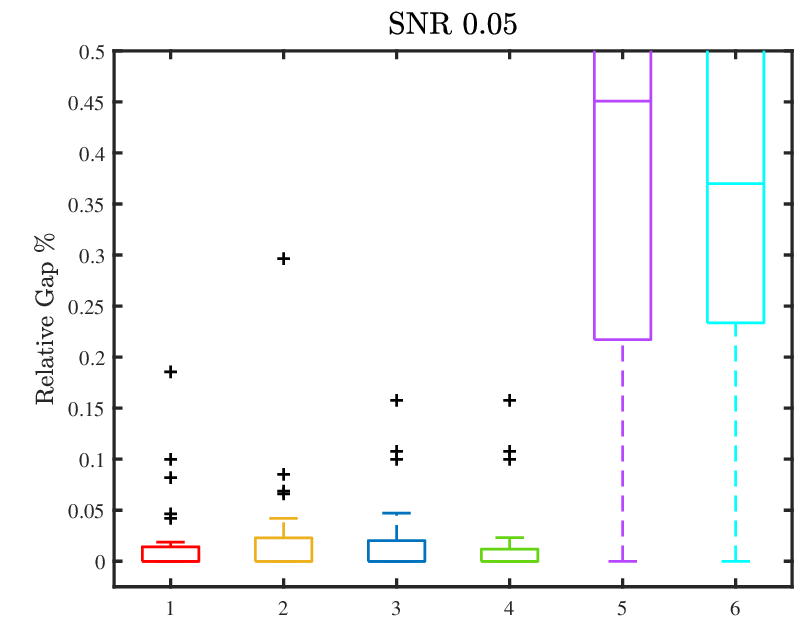}
     }
     \hfill
     \subfloat{\includegraphics[width=0.32\textwidth]{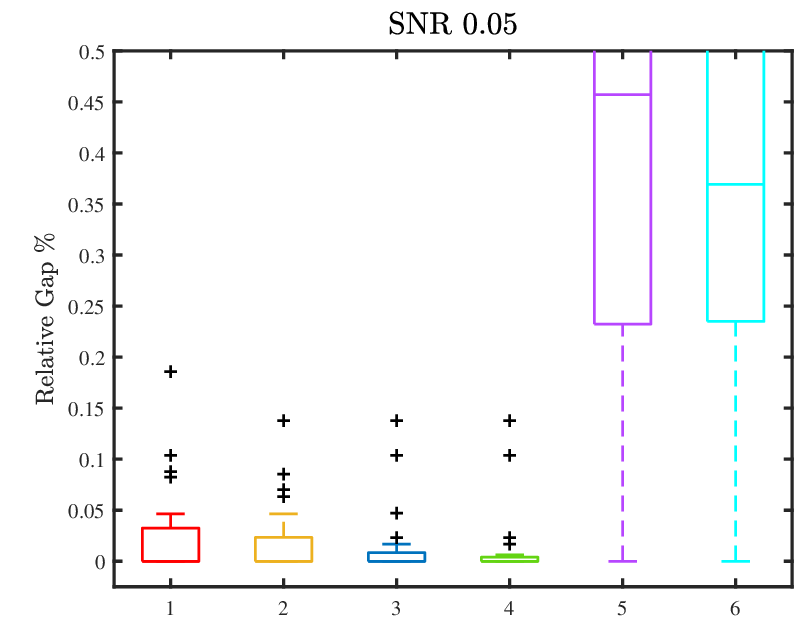}
     }
     \hfill
     \subfloat{\includegraphics[width=0.32\textwidth]{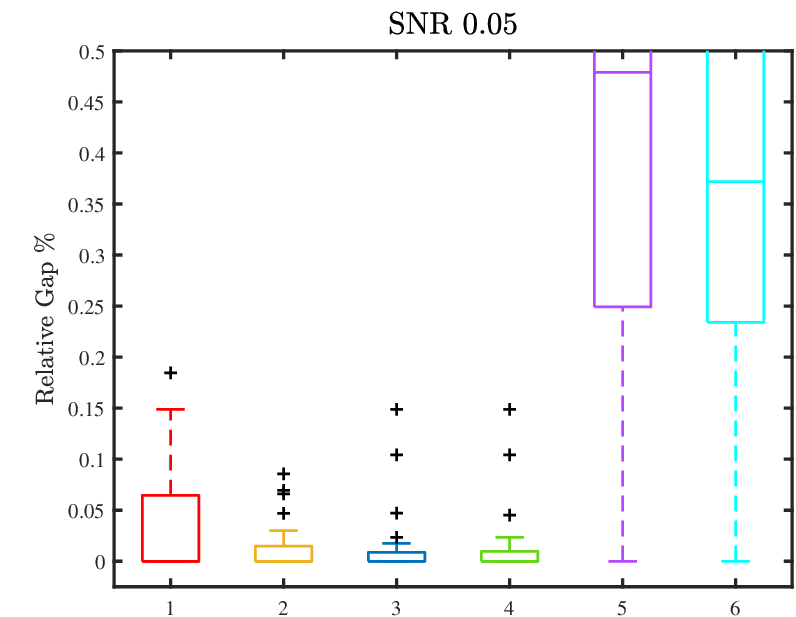}
     }
     
    \subfloat{\includegraphics[width=0.32\textwidth]{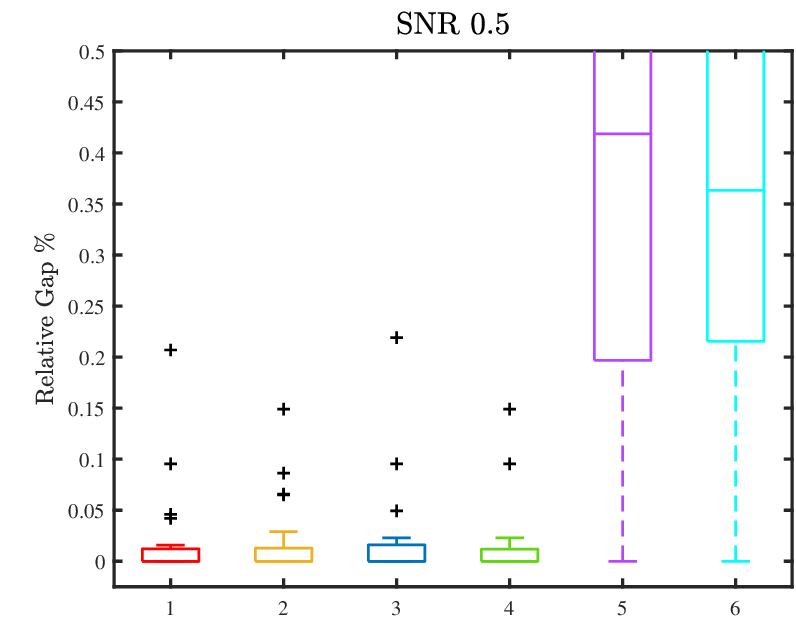}
     }
     \hfill
     \subfloat{\includegraphics[width=0.32\textwidth]{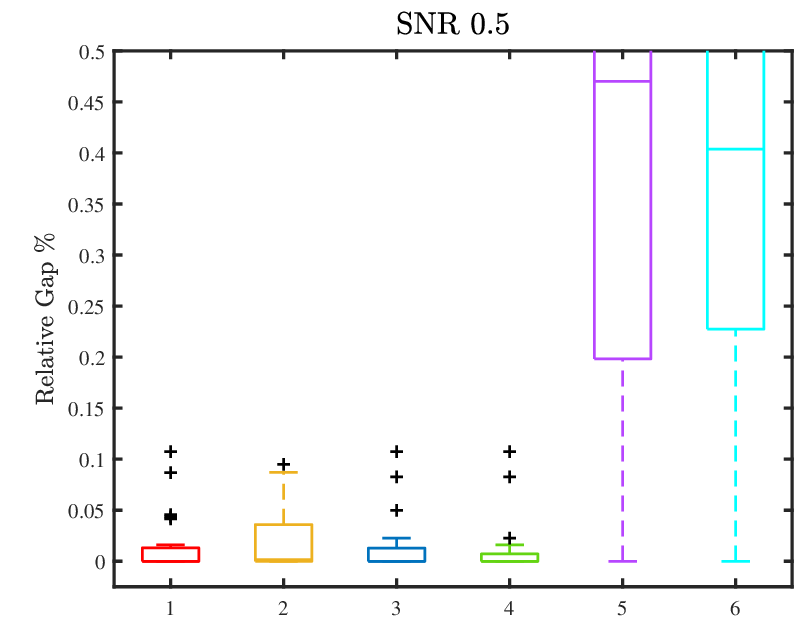}
     }
     \hfill
     \subfloat{\includegraphics[width=0.32\textwidth]{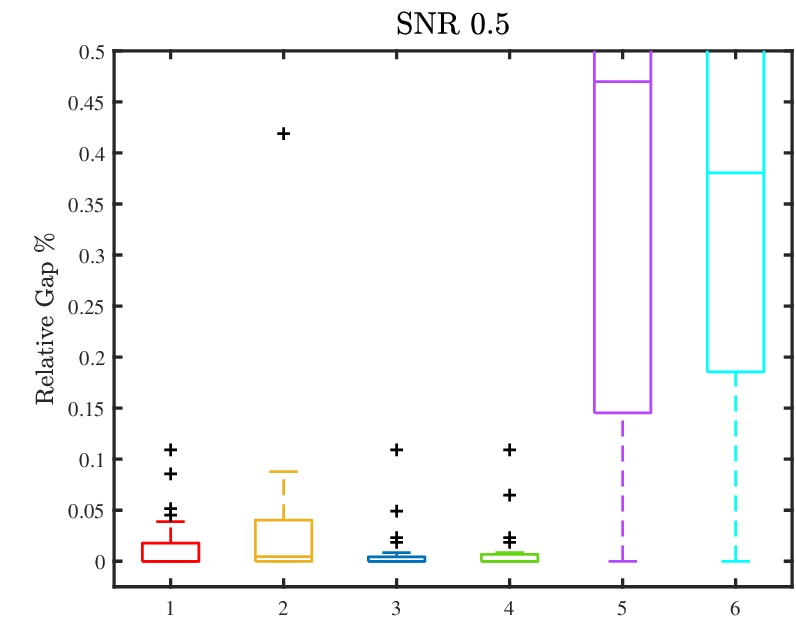}
     }
     
    \subfloat{\includegraphics[width=0.32\textwidth]{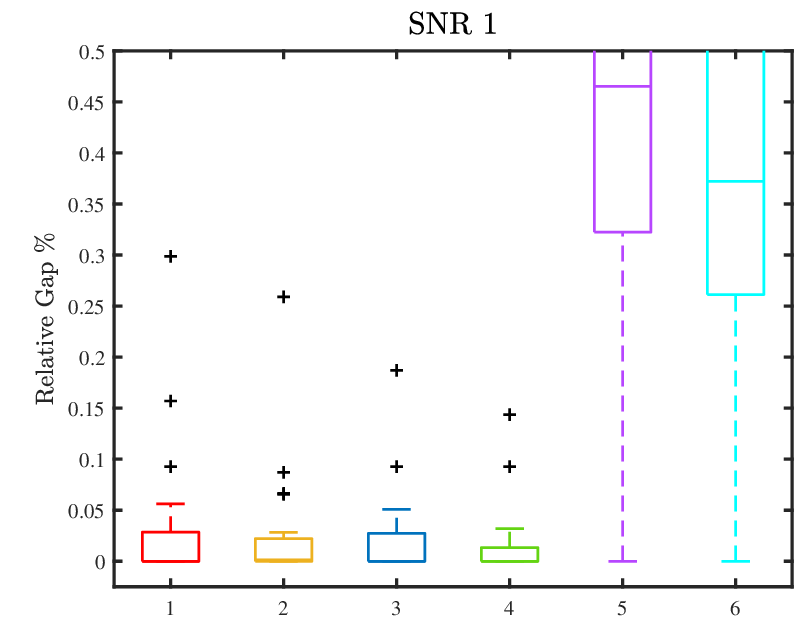}
     }
     \hfill
     \subfloat{\includegraphics[width=0.32\textwidth]{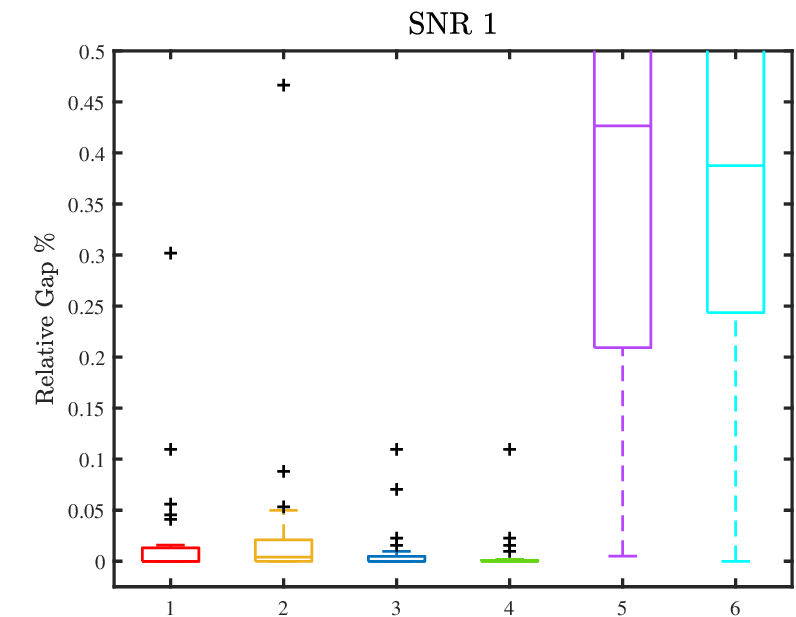}
     }
     \hfill
     \subfloat{\includegraphics[width=0.32\textwidth]{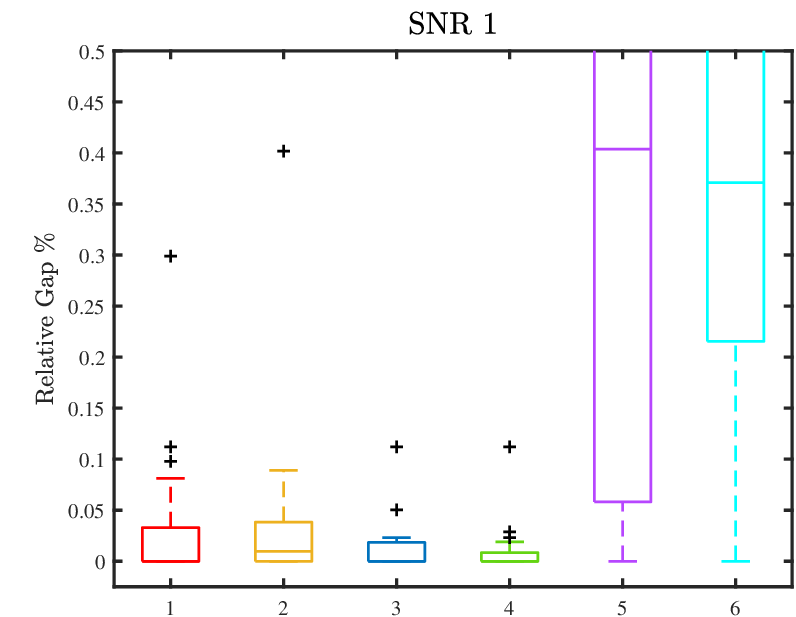}
     }
     
    \subfloat{\includegraphics[width=0.32\textwidth]{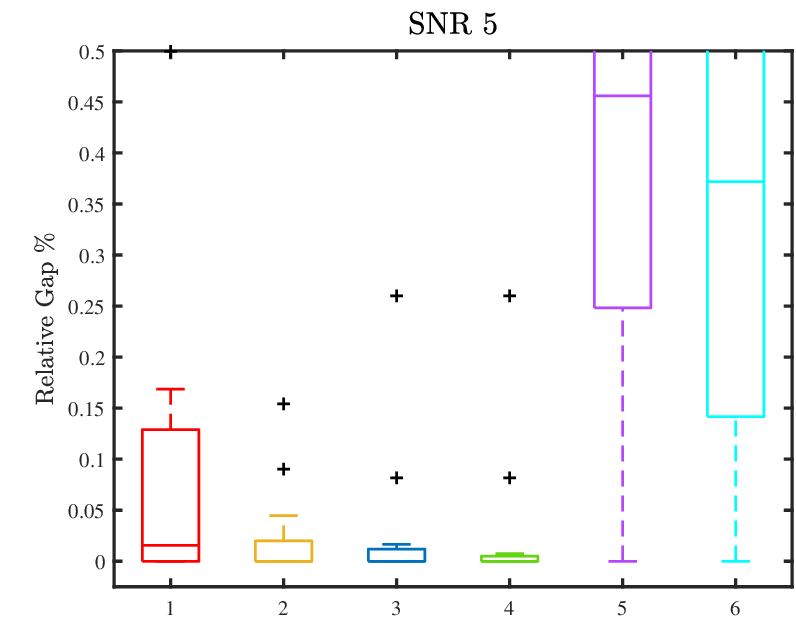}
     }
     \hfill
     \subfloat{\includegraphics[width=0.32\textwidth]{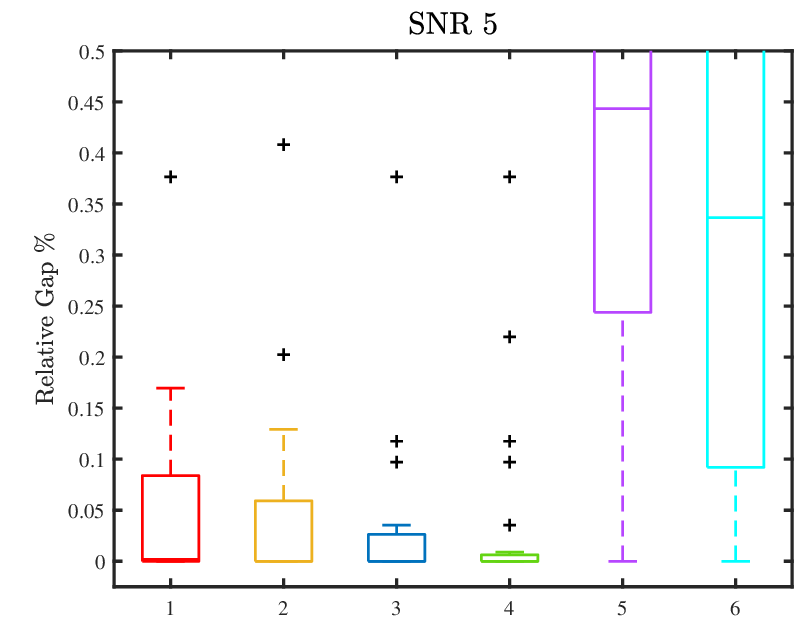}
     }
     \hfill
     \subfloat{\includegraphics[width=0.32\textwidth]{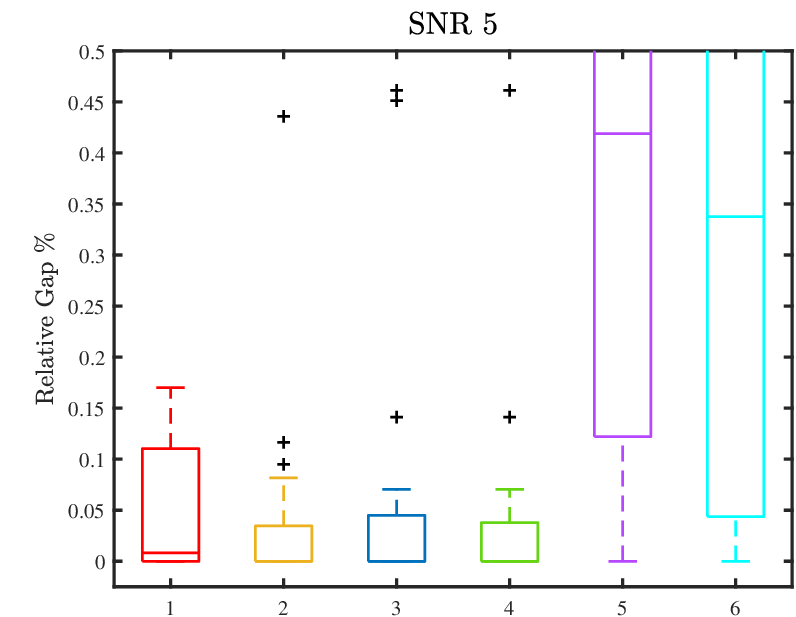}
     }
    \caption{Box plots of Relative Gap $\%$ for examples 1 (left), 2 (middle), and 3 (right) in OD case with constant correlation in small, medium, and large dimension regimes with four SNR values and the two $k$ values where (1) SFS1; (2) SFS2; (3) FS; (4) SFFS; (5) GA; (6) DFOn}
    \label{fig:boxplotsOdEgsRelGapSfsFsSffsGaDfowrtSNR}
\end{figure}

Figure \ref{fig:boxplotsUdEgsRelGapSfsFsSffsGaDfowrtSNR} shows box plots of the Relative Gap $\%$ for small, medium, and large dimension examples with data generated using constant correlation with four different SNR values in UD case. GA and DFOn are worse than the other four algorithms. However, in UD case, GA becomes competitive with DFOn by providing a slightly smaller median value for a few instances and with less variance in the output for every instance than DFOn. Among the remaining algorithms, for SNR 0.05, SFFS is the best, followed by SFS2, SFS1, FS; for the remaining three SNR values, SFS2 is the best, followed by SFFS, SFS1, FS. Thus, on average, SFS2 is considered as the overall best algorithm. 
% The results below are generated from the output
%/Users/vikra/Library/CloudStorage/Box-Box/LLS/testingLLS/outOct/Ubxtype12RCc_6alg/combinedOD/perprof26Dec24
%C:\Users\vikra\Desktop\testLLS\outLLSpackage05July\Ubxtype12R_6algSubOpt\ for
% box plots for Relative Gap % for 99b3e9c   and \outLLSpackage10Oct24\Ubx9d3fCcV9\  for 9d3f
\begin{figure}[h!]
    \subfloat{\includegraphics[width=0.32\textwidth]{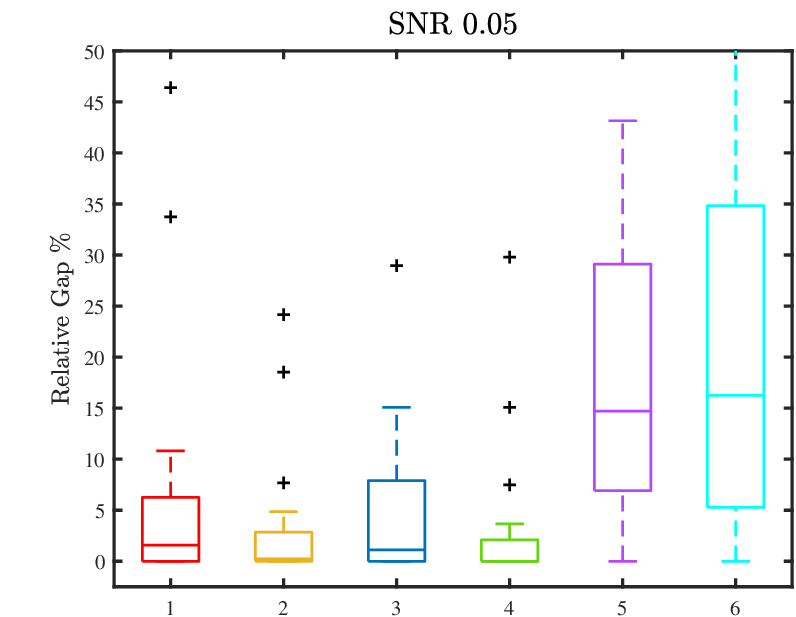}
     }
     \hfill
     \subfloat{\includegraphics[width=0.32\textwidth]{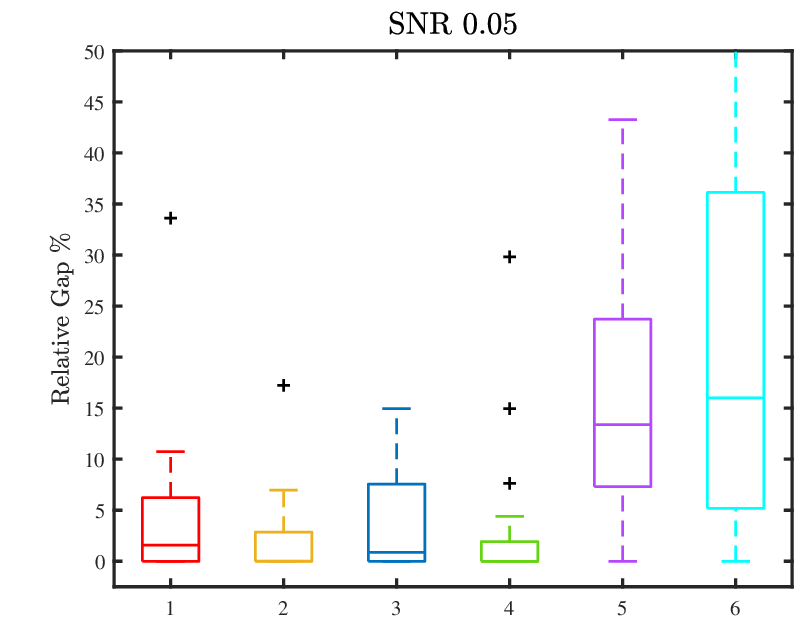}
     }
     \hfill
     \subfloat{\includegraphics[width=0.32\textwidth]{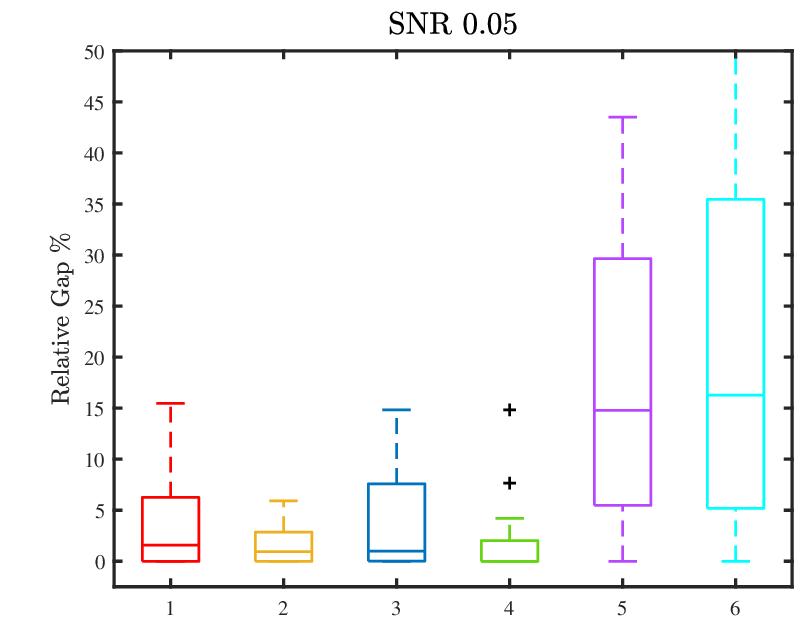}
     }
     
    \subfloat{\includegraphics[width=0.32\textwidth]{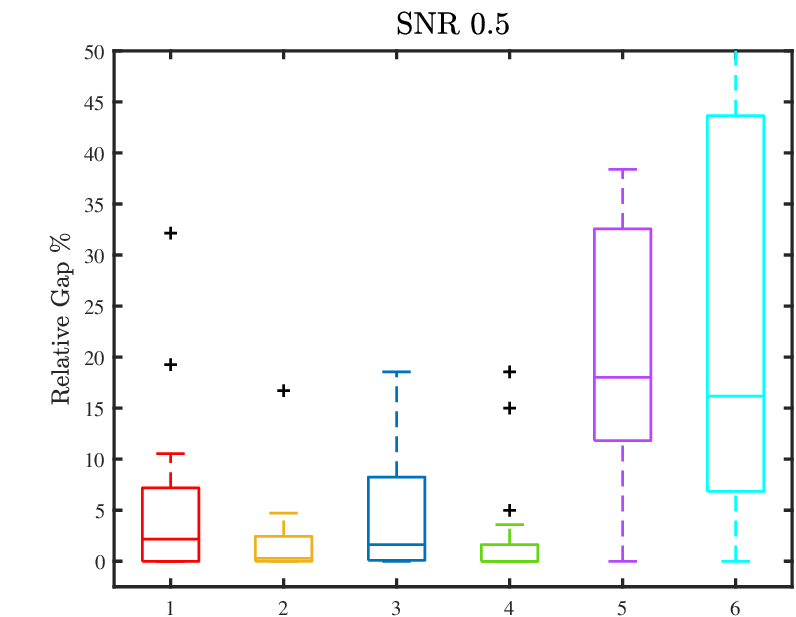}
     }
     \hfill
     \subfloat{\includegraphics[width=0.32\textwidth]{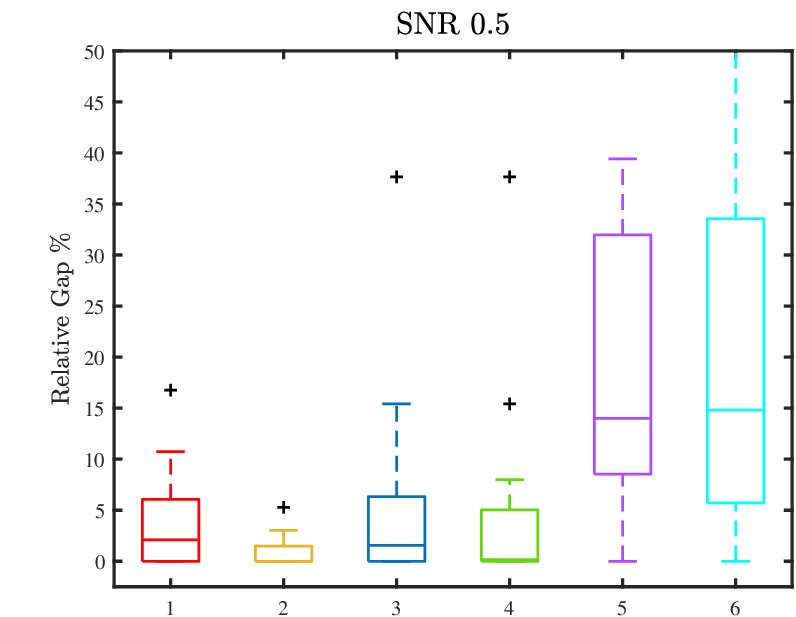}
     }
     \hfill
     \subfloat{\includegraphics[width=0.32\textwidth]{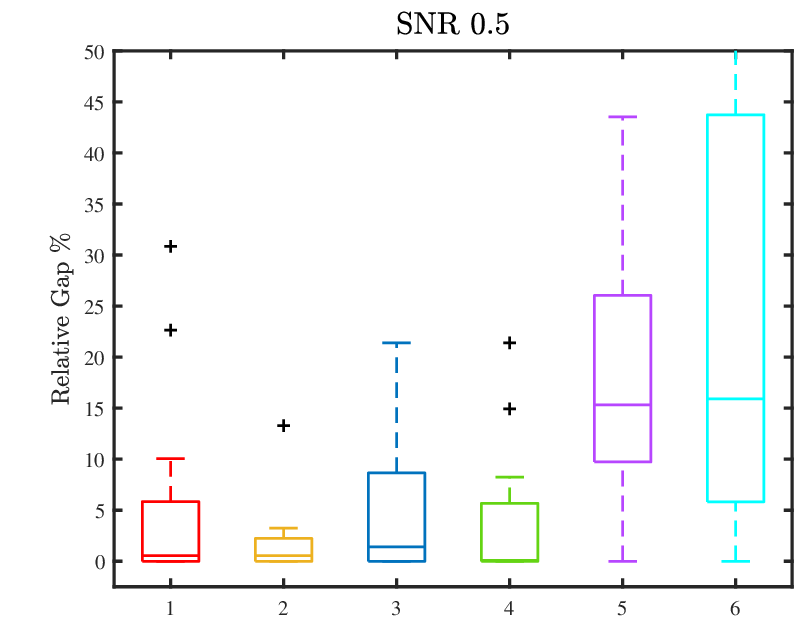}
     }
     
    \subfloat{\includegraphics[width=0.32\textwidth]{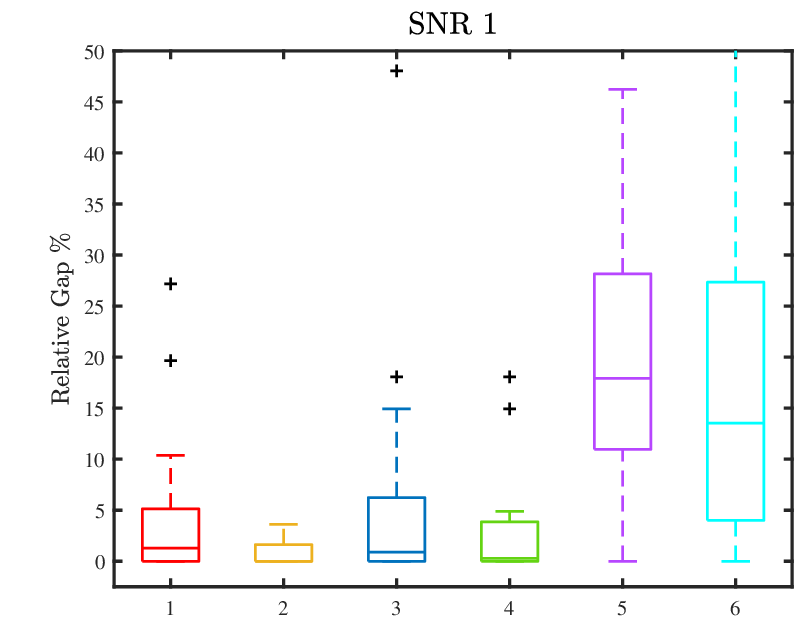}
     }
     \hfill
     \subfloat{\includegraphics[width=0.32\textwidth]{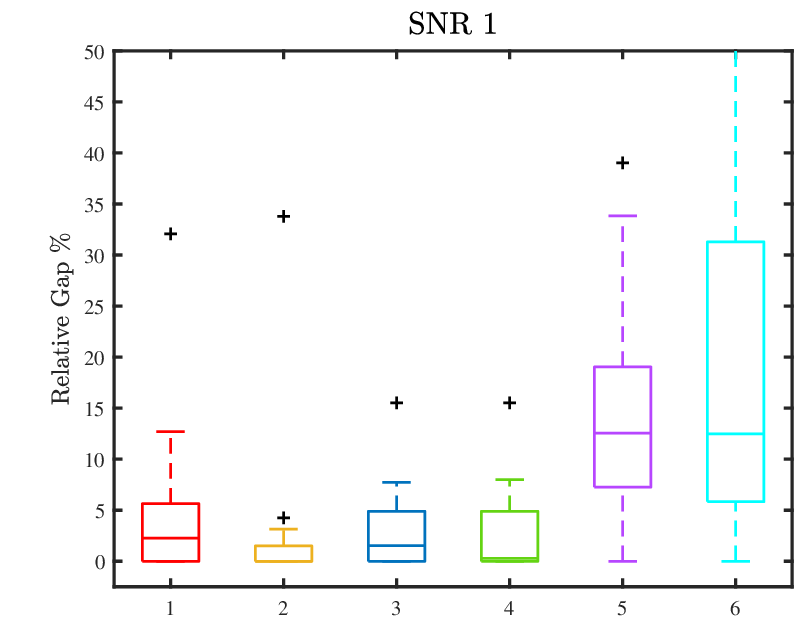}
     }
     \hfill
     \subfloat{\includegraphics[width=0.32\textwidth]{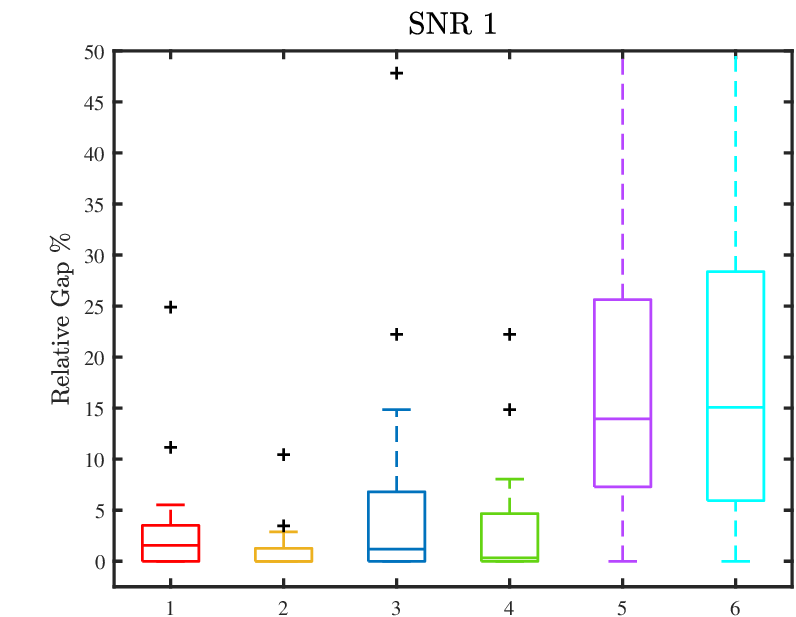}
     }
     
    \subfloat{\includegraphics[width=0.32\textwidth]{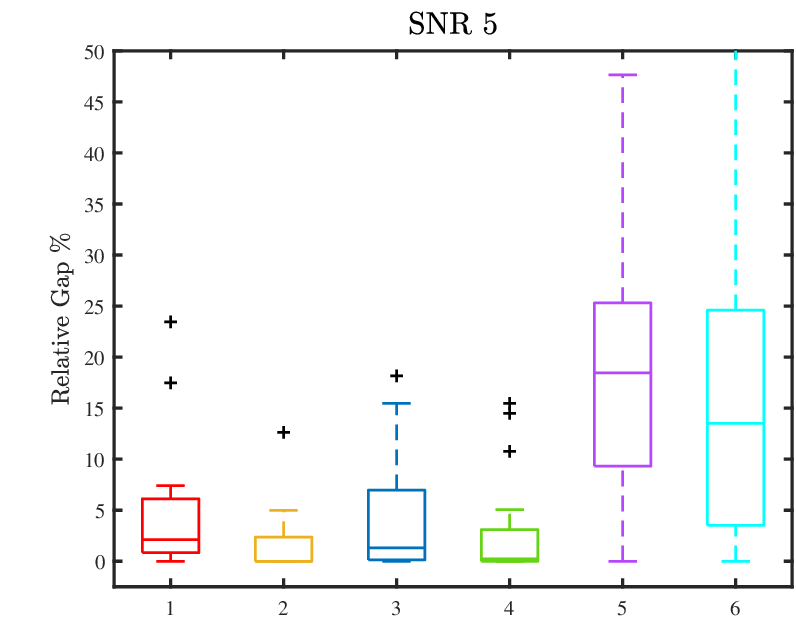}
     }
     \hfill
     \subfloat{\includegraphics[width=0.32\textwidth]{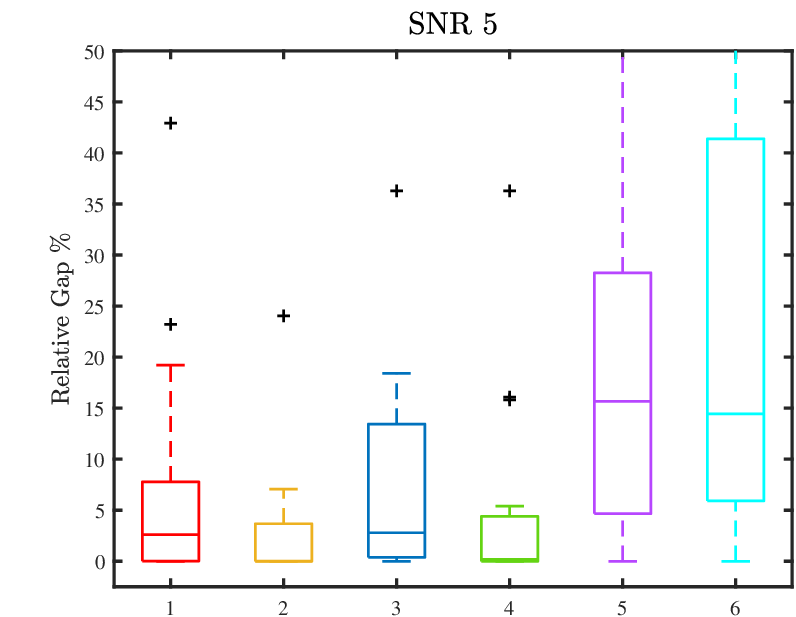}
     }
     \hfill
     \subfloat{\includegraphics[width=0.32\textwidth]{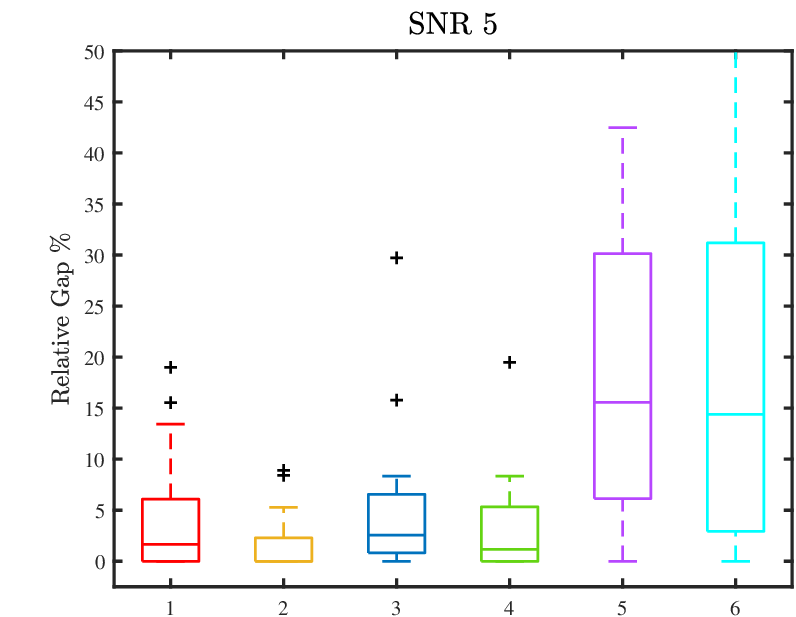}
     }
    \caption{Box plots of Relative Gap $\%$ for example 1 (left), 2 (middle), and 3 (right) in UD case with constant correlation in small, medium, and large dimension regimes with four SNR values and the two $k$ values where (1) SFS1; (2) SFS2; (3) FS; (4) SFFS; (5) GA; (6) DFOn}
    \label{fig:boxplotsUdEgsRelGapSfsFsSffsGaDfowrtSNR}
\end{figure}
\paragraph{CPU time performance}
Tables \ref{tab:cc_ODegs_avgcpu} and \ref{tab:cc_UDegs_avgcpu} show the average CPU time taken by an algorithm for examples 1, 2, and 3, with both $k$ values and in all the three-dimensional regimes. GA and SFS2 are computationally expensive, with GA taking the most CPU time in general. However, CPU time for SFS2 increases significantly for large-3 and large-4 dimension examples, making it the worst among all the other algorithms.
%C:\Users\vikra\Desktop\testLLS\outLLSpackage05July\Ubxtype12R_6algSubOpt\combined\UDegs  for 99b3e9c   and \outLLSpackage10Oct24\Ubx9d3fCcV9\  for 9d3f. The following data is for all the examples no matter the stop flag
\begin{table}[h]
    \centering
    \caption{Average CPU time taken in seconds for each dimension type by an algorithm for examples 1,2 and 3 in OD case with constant correlation}
    \label{tab:cc_ODegs_avgcpu}
    \begin{tabular}{llllllll}
    \hline\noalign{\smallskip}
     Type  &  SFS1  &  SFS2  &  FS   &  SFFS &  GA & DFO & DFOn\\
    \noalign{\smallskip}\hline\noalign{\smallskip}
    small-1 & 0.02 &  0.06 &  0.02 &  0.05 &  14.02 &  0.01 & 0.13\\
small-2 & 0.01 &  0.17 &  0.02 &  25.03 &  18.63 &  0.00 & 0.10\\
small-3 & 0.01 &  0.16 &  0.05 &  0.06 &  41.32 &  0.01 & 0.09\\
small-4 & 0.03 &  0.43 &  0.04 &  0.07 &  306.34 &  0.00 & 0.11\\
    \noalign{\smallskip}\hline\noalign{\smallskip}
    medium-1 & 0.15 &  4.69 &  0.14 &  0.15 &  171.99 &  0.01 & 0.15\\
medium-2 & 0.17 &  11.73 &  0.24 &  150.17 &  327.89 &  0.01 & 0.18\\
medium-3 & 0.19 &  15.30 &  0.26 &  0.46 &  343.14 &  0.02 & 0.30\\
medium-4 & 0.19 &  31.95 &  0.28 &  0.28 &  343.41 &  0.02 & 0.35\\
    \noalign{\smallskip}\hline\noalign{\smallskip}
   large-1 & 0.51 &  93.93 &  0.54 &  0.46 &  332.99 &  0.04 & 0.75\\
large-2 & 0.57 &  96.07 &  0.68 &  0.73 &  340.58 &  0.05 & 1.10\\
large-3 & 0.75 &  337.78 &  1.06 &  0.99 &  241.26 &  0.11 & 2.27\\
large-4 & 1.19 &  563.18 &  1.22 &  1.76 &  297.62 &  0.21 & 4.39\\
    \noalign{\smallskip}\hline
    \end{tabular}
\end{table}
\begin{table}[h]
    \centering
    \caption{Average CPU time taken in seconds for each dimension type by an algorithm for examples 1,2 and 3 in UD case with constant correlation}
    \label{tab:cc_UDegs_avgcpu}
    \begin{tabular}{llllllll}
    \hline\noalign{\smallskip}
     Type  &  SFS1    &  SFS2    &  FS   &  SFFS &  GA & DFO & DFOn \\
    \noalign{\smallskip}\hline\noalign{\smallskip}
    small-1 & 0.02 &  0.06 &  0.02 &  225.02 &  7.55 &  0.01 & 0.11\\
small-2 & 0.02 &  0.18 &  0.02 &  0.03 &  3.95 &  0.01 & 0.11\\
small-3 & 0.03 &  0.34 &  0.05 &  25.05 &  21.53 &  0.00 & 0.11\\
small-4 & 0.03 &  0.61 &  0.04 &  0.07 &  27.66 &  0.00 & 0.10\\
    \noalign{\smallskip}\hline\noalign{\smallskip}
    medium-1 & 0.14 &  4.33 &  0.16 &  0.18 &  378.38 &  0.01 & 0.14\\
medium-2 & 0.16 &  12.12 &  0.23 &  0.21 &  343.94 &  0.01 & 0.18\\
medium-3 & 0.27 &  18.39 &  0.25 &  0.36 &  75.03 &  0.02 & 0.30\\
medium-4 & 0.20 &  23.52 &  0.28 &  0.42 &  280.49 &  0.02 & 0.36\\
    \noalign{\smallskip}\hline\noalign{\smallskip}
large-1 & 0.34 &  114.83 &  0.51 &  100.61 &  132.87 &  0.04 & 0.78\\
large-2 & 0.65 &  136.14 &  0.72 &  50.79 &  301.59 &  0.05  & 1.07\\
large-3 & 1.18 &  364.06 &  1.08 &  125.98 &  199.85 &  0.11 & 2.26\\
large-4 & 0.98 &  507.43 &  1.21 &  126.38 &  227.82 &  0.21 & 4.28\\
    \noalign{\smallskip}\hline
    \end{tabular}
\end{table}

Figures \ref{fig:perprofEgs12ROdCpuSfsFsSffsGaDfo} and \ref{fig:perprofEgs12RUdCpuSfsFsSffsGaDfo} show performance profiles of CPU time for examples 1, 2, and 3 with constant correlation in all the three dimension regimes for the OD and the UD case, respectively, for those examples for which all the algorithms terminated without reaching any hard stop. For clarity of the whole figure, the GA has been excluded from all the performance profiles, and in each of the medium and large dimensions, the SFS2-curve has also been removed. DFO is the fastest. SFS1, FS, and SFFS are close, with SFS1 slightly better than FS and SFFS; FS is slightly better than SFFS. DFOn, being the sum of the CPU time of all the n runs, is worse than SFS1, FS, SFFS, and DFO algorithms.
% combined data for Examples 1, 2, and 3 to create perprof only for the examples for which all the algorithms fully converged
%C:\Users\vikra\Desktop\testLLS\outLLSpackage05July\Ubxtype12R_6algSubOpt\combined\ODegs  for 99b3e9c   and \outLLSpackage10Oct24\Ubx9d3fCcV9\  for 9d3f
\begin{figure}[ht]
    \subfloat{\includegraphics[width=0.32\textwidth]{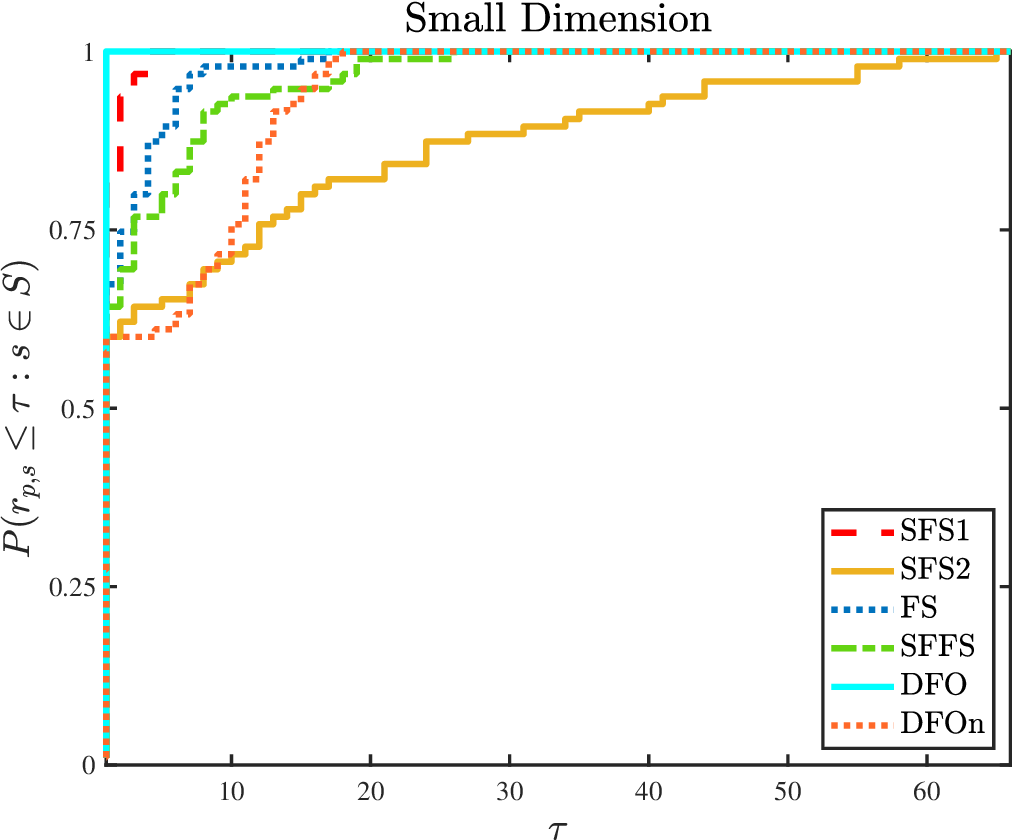}}
    \hfill
    \subfloat{\includegraphics[width=0.32\textwidth]{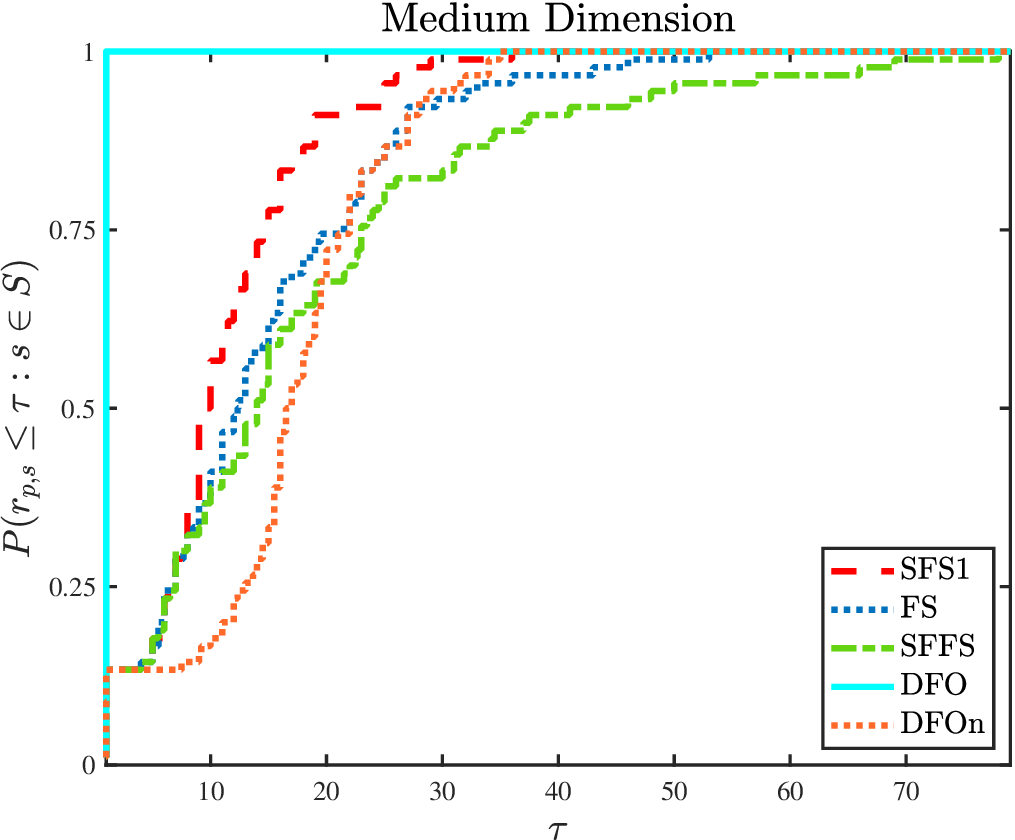}}
    \hfill
    \subfloat{\includegraphics[width=0.32\textwidth]{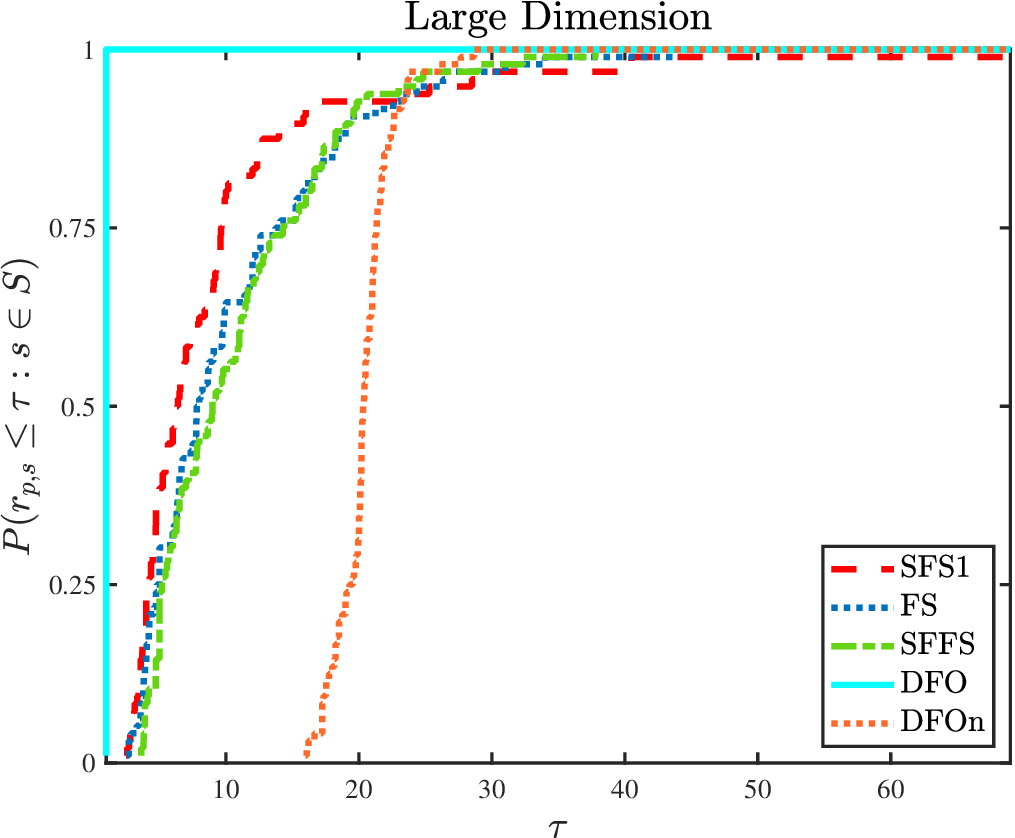}}
    \caption{Performance profiles of CPU time for examples 1, 2, and 3 combined in OD case with constant correlation in small, medium, and large dimension regimes with four SNR values and the two $k$ values}
    \label{fig:perprofEgs12ROdCpuSfsFsSffsGaDfo}
\end{figure}
%C:\Users\vikra\Desktop\testLLS\outLLSpackage05July\Ubxtype12R_6algSubOpt\combined\UDegs  for 99b3e9c   and \outLLSpackage10Oct24\Ubx9d3fCcV9\  for 9d3f
\begin{figure}[h]
    \subfloat{\includegraphics[width=0.32\textwidth]{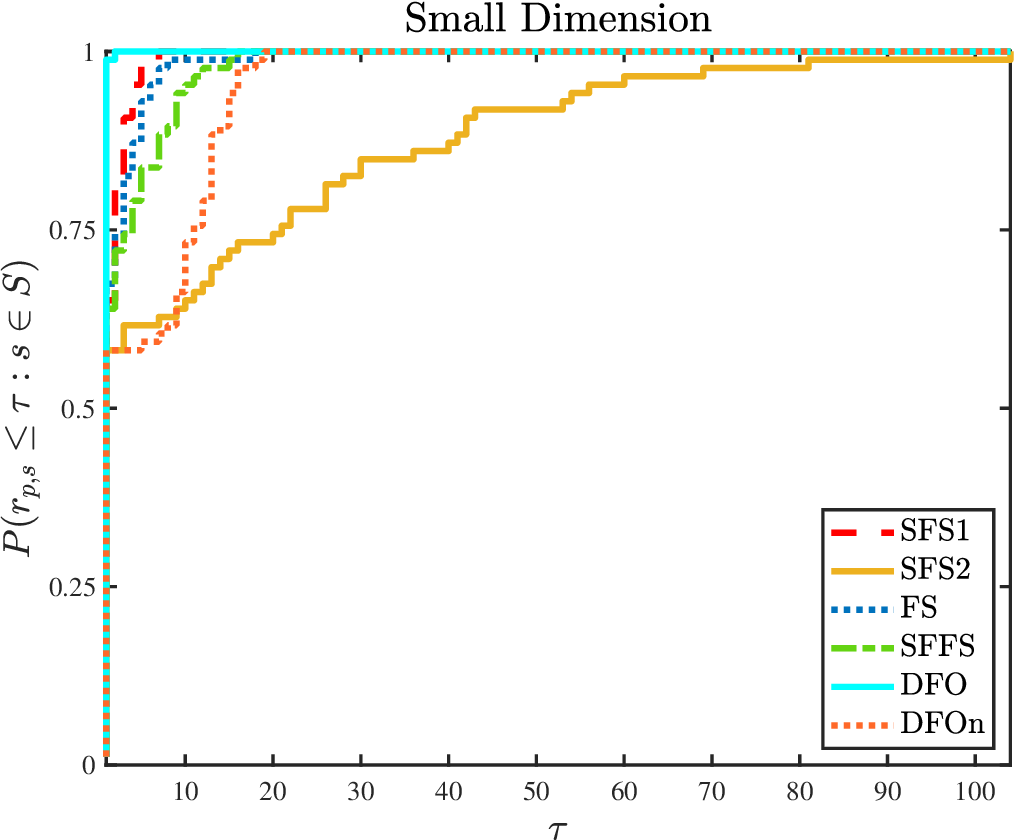}}
    \hfill
    \subfloat{\includegraphics[width=0.32\textwidth]{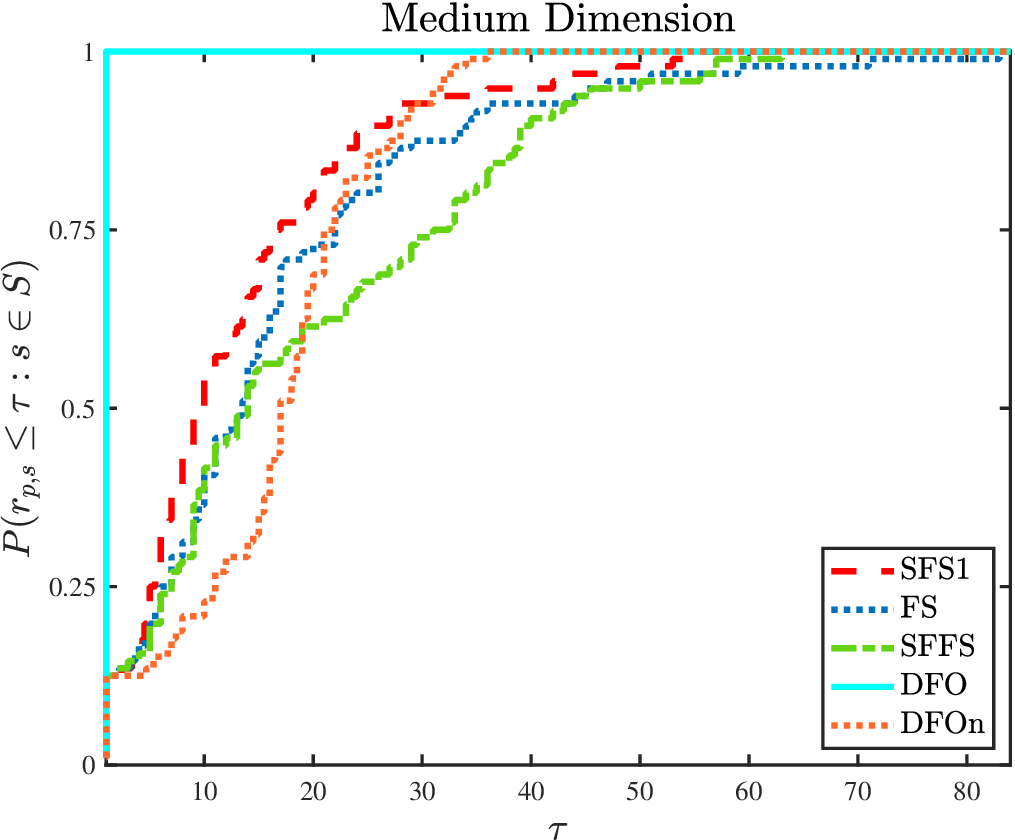}}
    \hfill
    \subfloat{\includegraphics[width=0.32\textwidth]{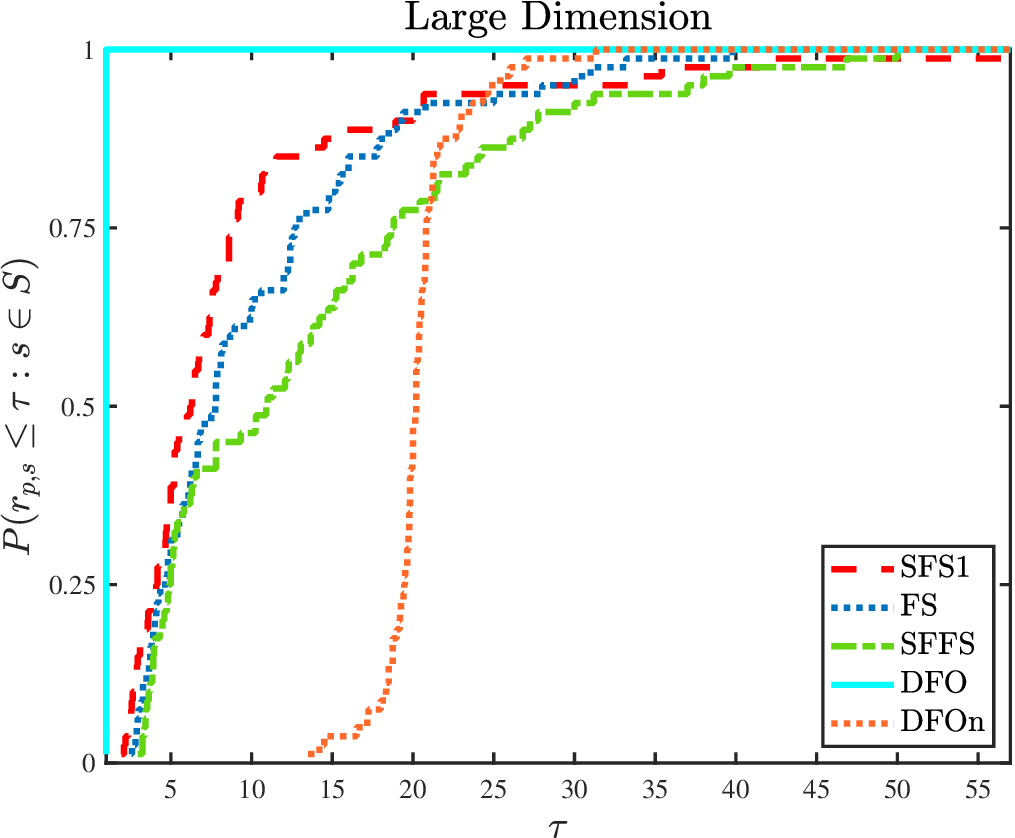}}
    \caption{Performance profiles of CPU time for examples 1, 2, and 3 combined in UD case with constant correlation in small, medium, and large dimension regimes with four SNR values and the two $k$ values}
    \label{fig:perprofEgs12RUdCpuSfsFsSffsGaDfo}
\end{figure}
\begin{table}[h]
    \centering
    \caption{Number of instances for examples 1, 2, and 3 in which algorithms stopped with CPU time or maximum iteration limit out of the total 32 instances for each type with constant correlation in small, medium, and large dimension regimes with four SNR values and the two $k$ values}
    \label{tab:cc_BDegs_hardstopcount}
    \begin{tabular}{lllllllllll}
    \hline\noalign{\smallskip}
    &Type & \multicolumn{3}{l}{Example 1} & \multicolumn{3}{l}{Example 2} & \multicolumn{3}{l}{Example 3}  \\
    \noalign{\smallskip}\hline\noalign{\smallskip}
     &    & SFS2 & SFFS & GA & SFS2 &  SFFS &  GA  & SFS2 & SFFS & GA\\
    \noalign{\smallskip}\hline\noalign{\smallskip}
   OD & small  &  0  &  0  &  10  &  0   &   0   &   9  &   0  &  1   &  10\\
   OD & medium &  0  &  1  &  32  &  0   &   3   &   31  &   0  &  2   &  31\\
   OD & large  &  4  &  0  &  29  &  2   &   0   &   28  &   1  &  0   &  31\\
    \noalign{\smallskip}\hline\noalign{\smallskip}
   UD & small  &  0  &   4   &  6 &  0   &   3   &   4  &  0   &   3  &   5\\
   UD & medium &  0  &   0   &  30 &  0   &   0   &  27  &  0   &   0  &   28\\
   UD & large  &  5  &   5   &  32 &  0   &   8   &   32  &  0   &   3  &   32\\
    \noalign{\smallskip}\hline
    \end{tabular}
\end{table}
Our results show that for every instance, SFS1, FS, and DFO converged (locally) and terminated without reaching the given hard-stopping limits. SFS2 and SFFS terminated prematurely by reaching the maximum CPU time limit provided, whereas GA got terminated using both the maximum iteration limit and maximum CPU time limit. Table \ref{tab:cc_BDegs_hardstopcount} shows these numbers for the three algorithms. SFFS may take more CPU time because SFFS does an additional check to drop previously selected predictors in the search for a better model, and this step can take a significant amount of time if there is more randomness in the data. This is the reason for a higher number of hard CPU time limit stops by SFFS in the UD case than in the OD case.

An algorithm is considered ideal if it ranks on the top in both solution quality and CPU time performance. Our numerical results did not reveal such an algorithm. Under the circumstances, SFS1 and SFS2 are considered good choices for suboptimal algorithms for solving \eqref{bsschp5}.
%=========================================================================
\subsubsection{Test results with exponential correlation data}
For the synthetic data sets with exponential correlation, the covariance matrix $\Sigma$ is chosen such that $\Sigma_{i,j}=0.8^{|i-j|}$ when $i\neq j$ and $\Sigma_{i,i}=1$.
%\subsection{Test results for data with exponential correlation}
\paragraph{Solution quality}
Figure \ref{fig:boxplotsOdEcEgsRelGapSfsFsSffsGaDfowrtSNR} shows box plots of the Relative Gap $\%$ for small, medium, and large dimension examples with data generated using exponential correlation with four different SNR values for OD case. GA is the worst. DFOn performs better than GA but is worse than the other algorithms. Among the remaining four algorithms, SFS2 is the best. There is no clear conclusion among SFS1, FS, and SFFS. 
% The results below are generated from the output 
% C:\Users\vikra\Desktop\testLLS\outLLSpackage05July\UbxEctype12R_6algSubOpt\ for 99b3e9c   and \outLLSpackage10Oct24\  for 9d3f
% box plots for Relative Gap % for small, med,large and large dim. OD egs for egs 1,2 and 3
\begin{figure}
    \centering
    \subfloat{\includegraphics[width=0.32\linewidth]{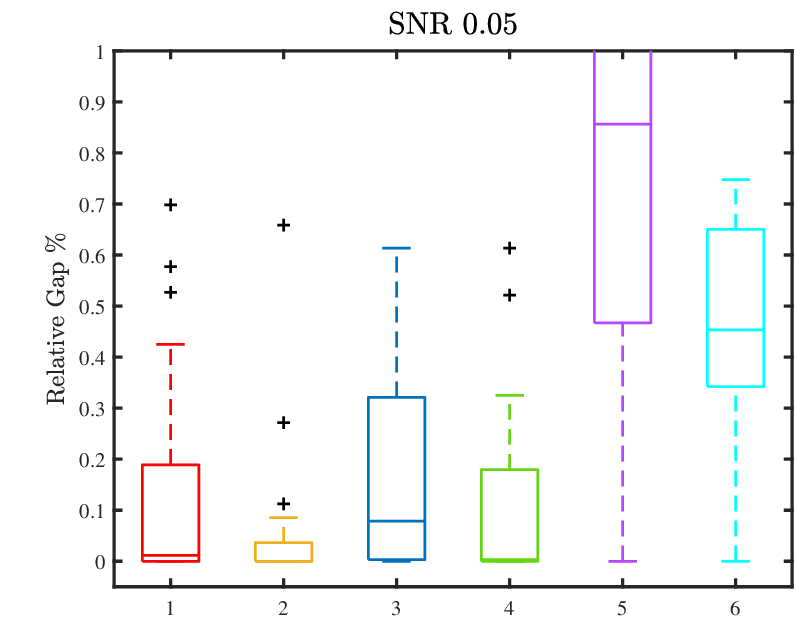}}
    \hfill
    \subfloat{\includegraphics[width=0.32\linewidth]{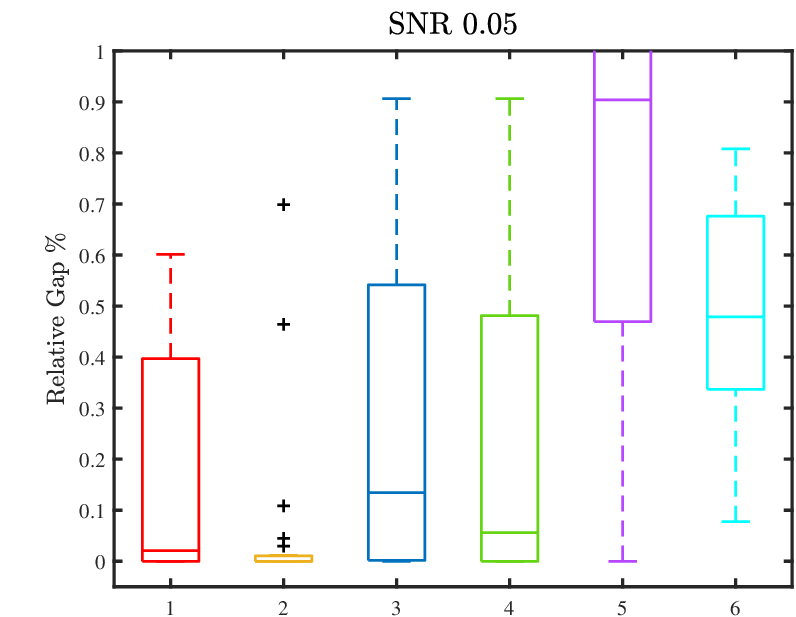}}
    \hfill
    \subfloat{\includegraphics[width=0.32\linewidth]{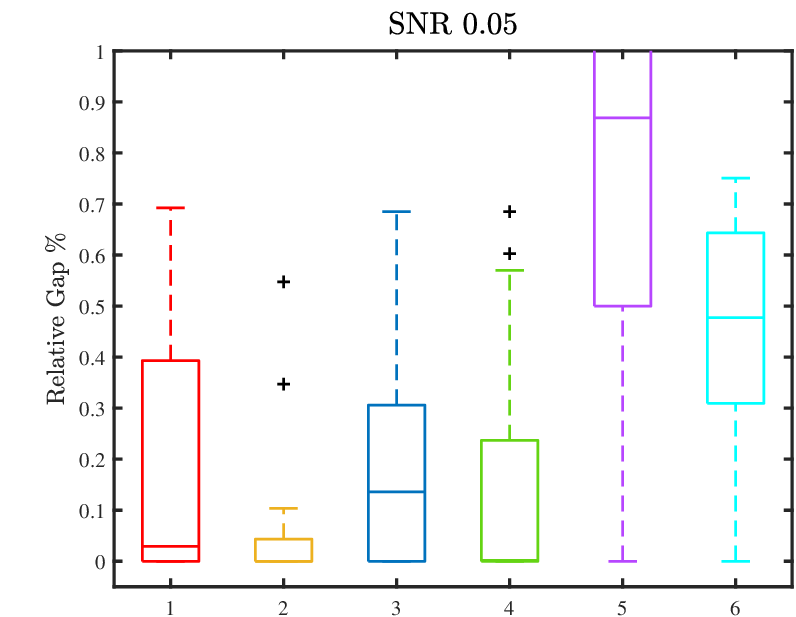}}

    \centering
    \subfloat{\includegraphics[width=0.32\linewidth]{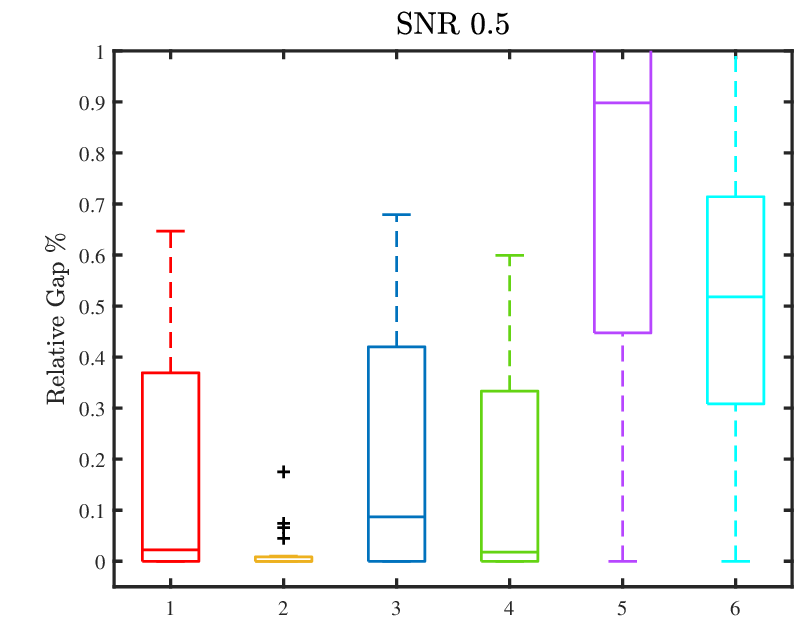}}
    \hfill
    \subfloat{\includegraphics[width=0.32\linewidth]{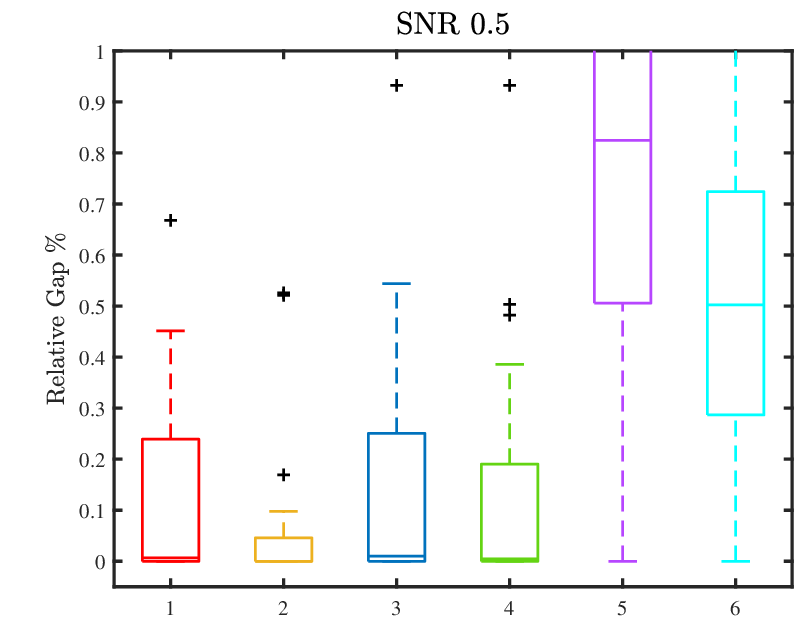}}
    \hfill
    \subfloat{\includegraphics[width=0.32\linewidth]{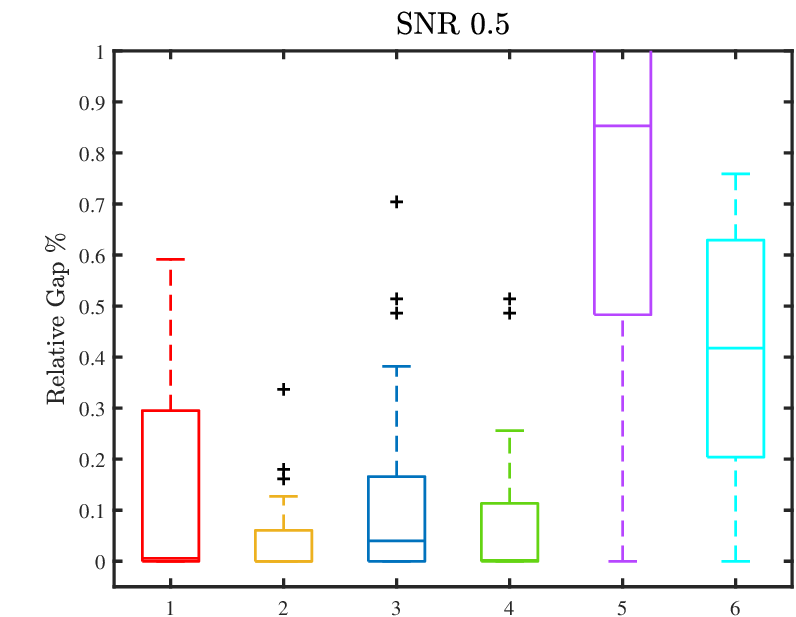}}

    \centering
    \subfloat{\includegraphics[width=0.32\linewidth]{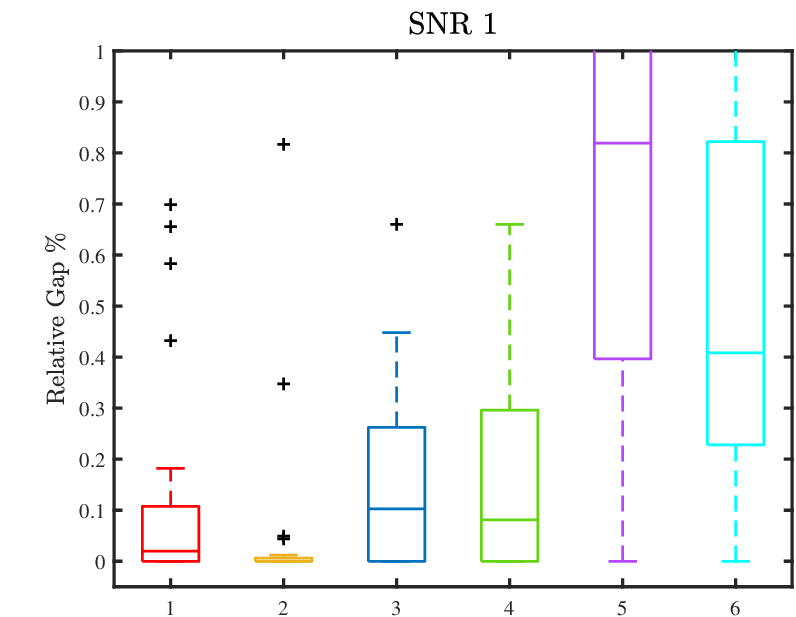}}
    \hfill
    \subfloat{\includegraphics[width=0.32\linewidth]{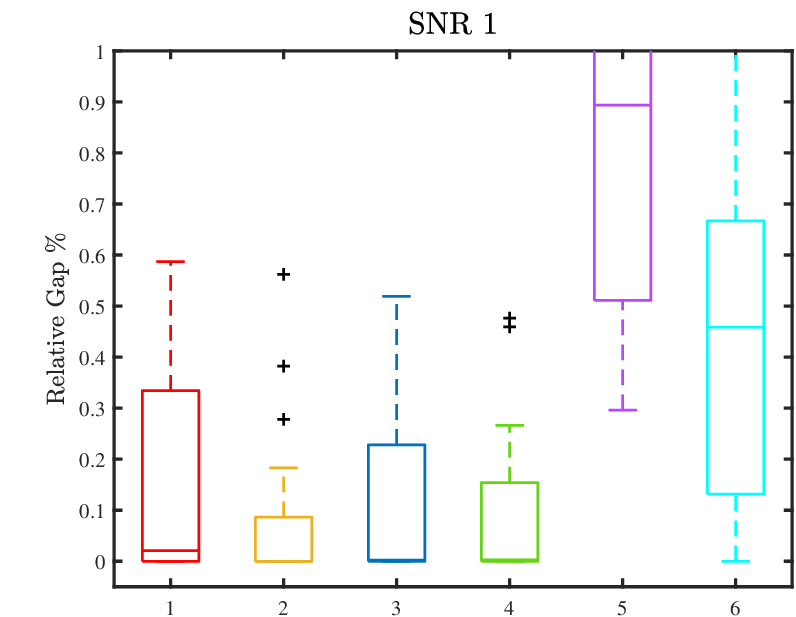}}
    \hfill
    \subfloat{\includegraphics[width=0.32\linewidth]{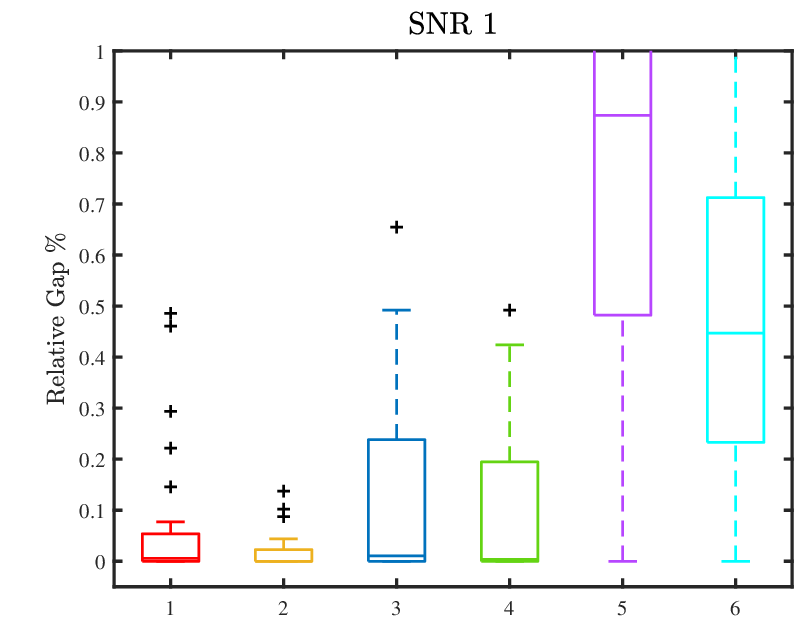}}

    \centering
    \subfloat{\includegraphics[width=0.32\linewidth]{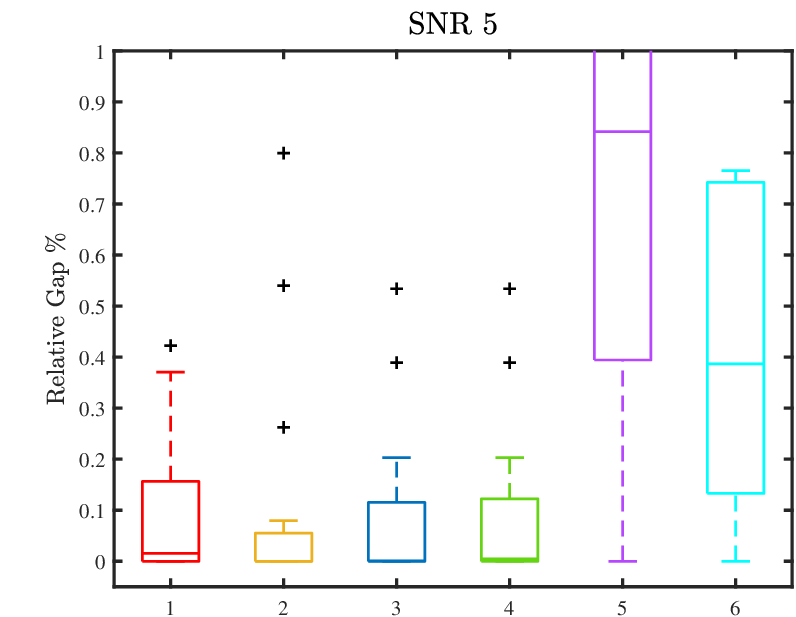}}
    \hfill
    \subfloat{\includegraphics[width=0.32\linewidth]{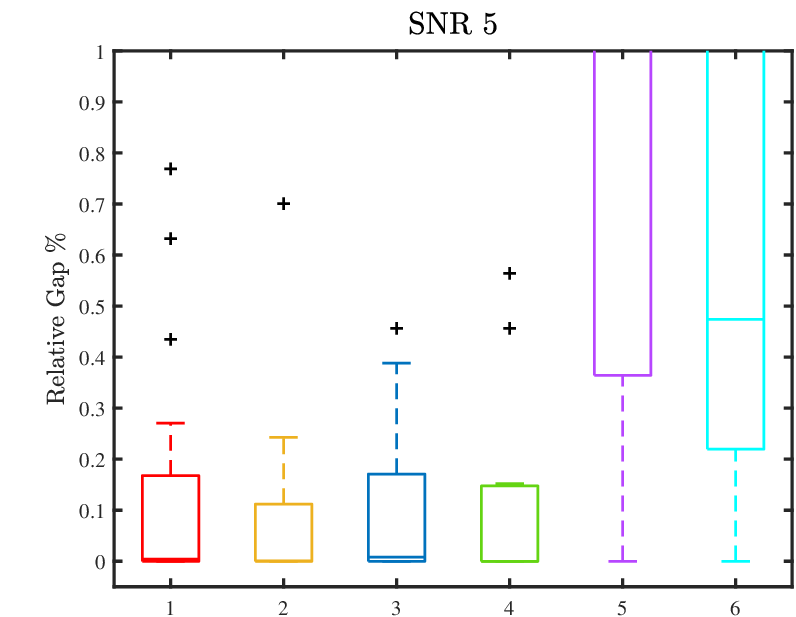}}
    \hfill
    \subfloat{\includegraphics[width=0.32\linewidth]{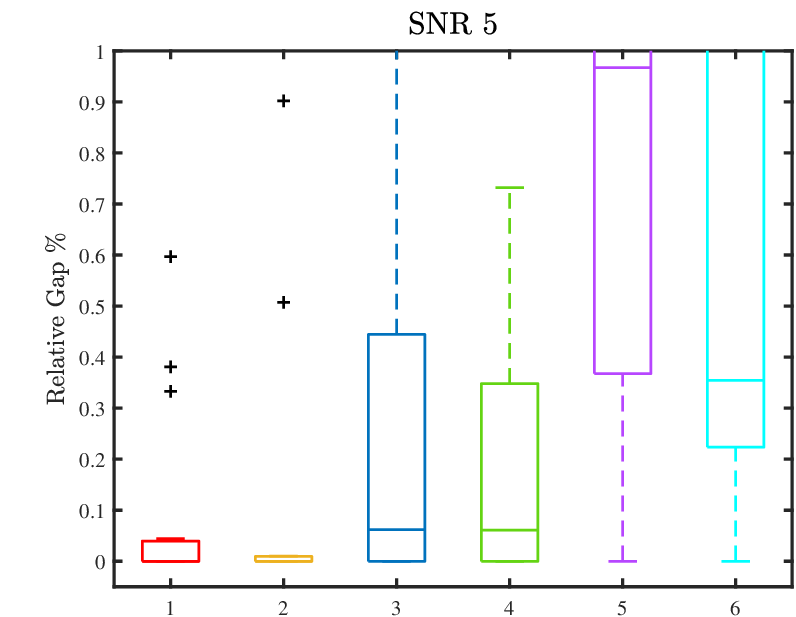}}
    \caption{Box plots of Relative Gap $\%$ for examples 1 (left), 2 (middle), and 3 (right) in OD case with exponential correlation in small, medium, and large dimension regimes with four SNR values and the two $k$ values where (1) SFS1; (2) SFS2; (3) FS; (4) SFFS; (5) GA; (6) DFOn}
    \label{fig:boxplotsOdEcEgsRelGapSfsFsSffsGaDfowrtSNR}
\end{figure} 

Figure \ref{fig:boxplotsUdEcEgsRelGapSfsFsSffsGaDfowrtSNR} shows box plots of the Relative Gap $\%$ for small, medium, and large dimension examples with data generated using exponential correlation with four different SNR values for UD case. There is no clear winner among SFS1, SFS2, and SFFS. But in general, SFS2 performs better than SFFS and SFS1, followed by SFFS and then SFS1. For the remaining three algorithms, FS is worse than SFS1, SFS2, and SFFS but performs better than DFOn and GA. DFOn performs better than GA.
% The results below are generated from the output 
% box plots for Relative Gap % for med,large and ultra large dim. OD egs for egs 1,2 and 3
\begin{figure}
    \centering
    \subfloat{\includegraphics[width=0.32\linewidth]{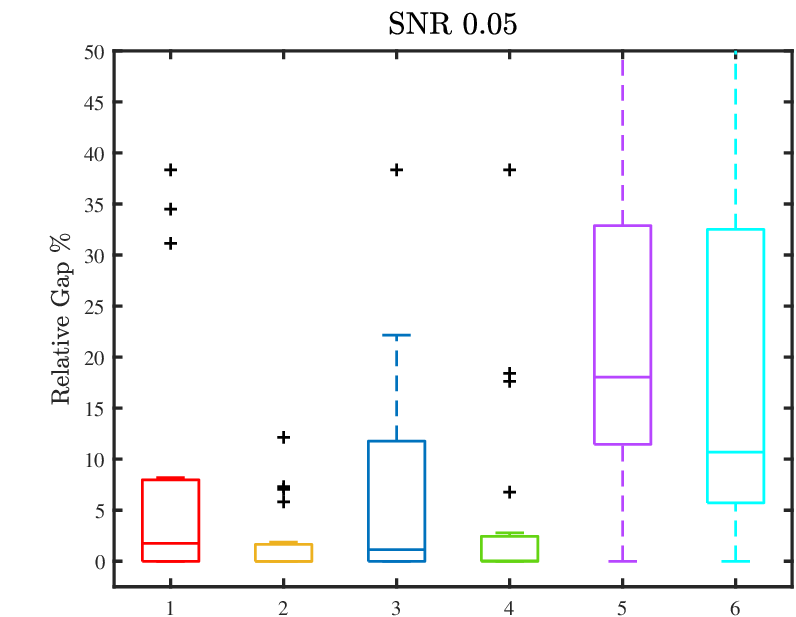}}
    \hfill
    \subfloat{\includegraphics[width=0.32\linewidth]{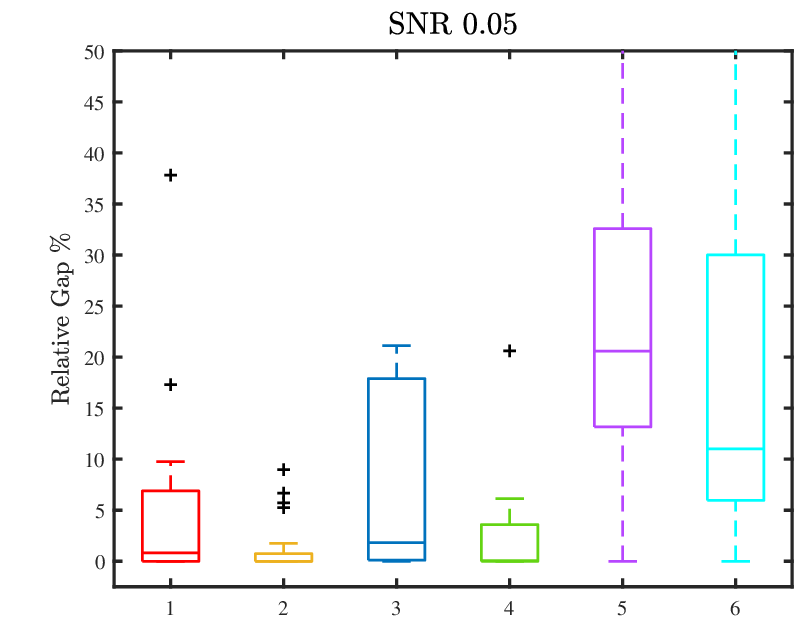}}
    \hfill
    \subfloat{\includegraphics[width=0.32\linewidth]{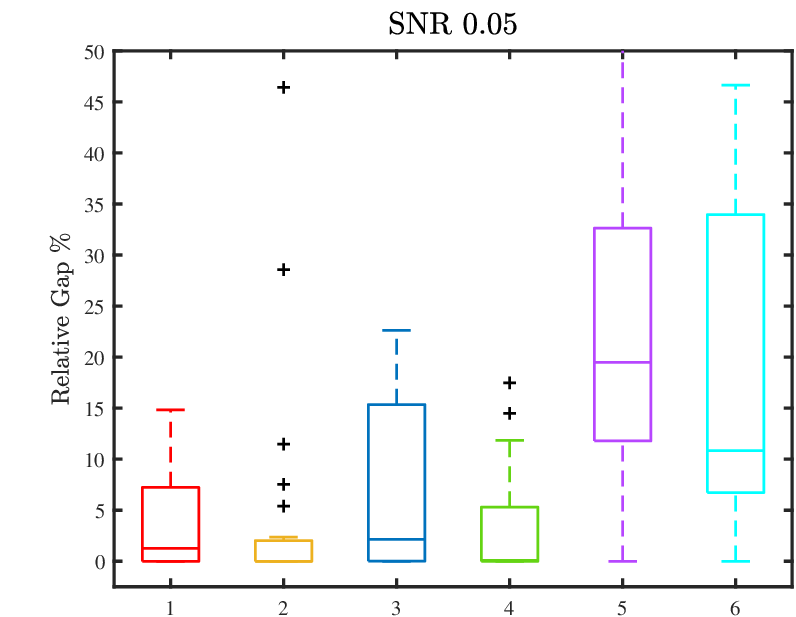}}

    \centering
    \subfloat{\includegraphics[width=0.32\linewidth]{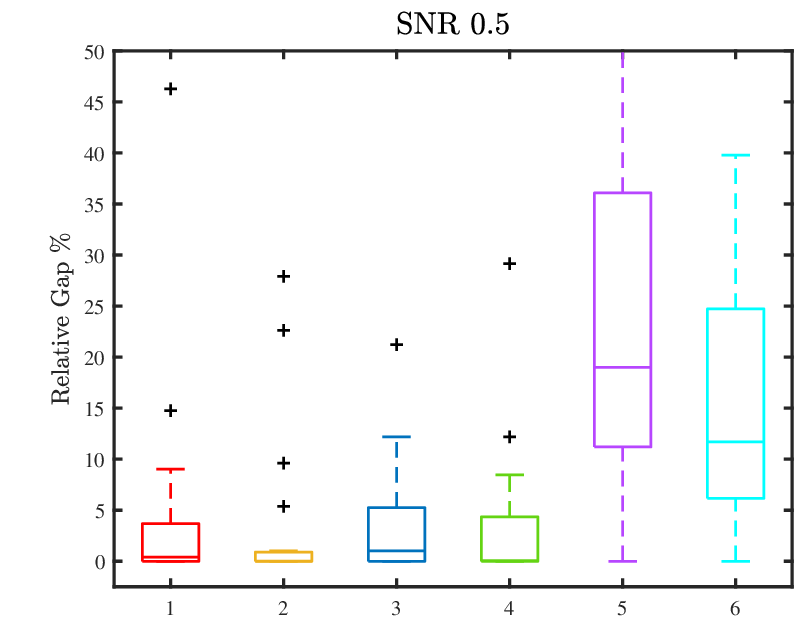}}
    \hfill
    \subfloat{\includegraphics[width=0.32\linewidth]{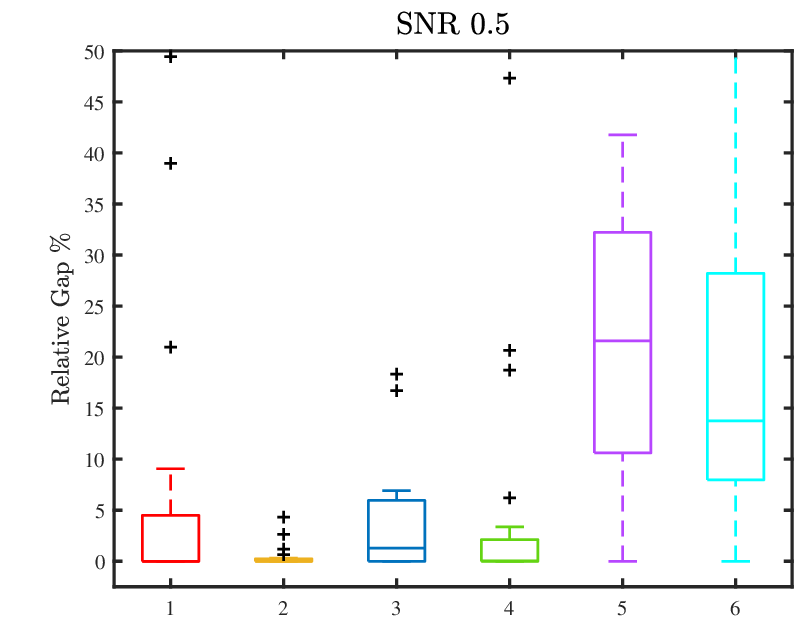}}
    \hfill
    \subfloat{\includegraphics[width=0.32\linewidth]{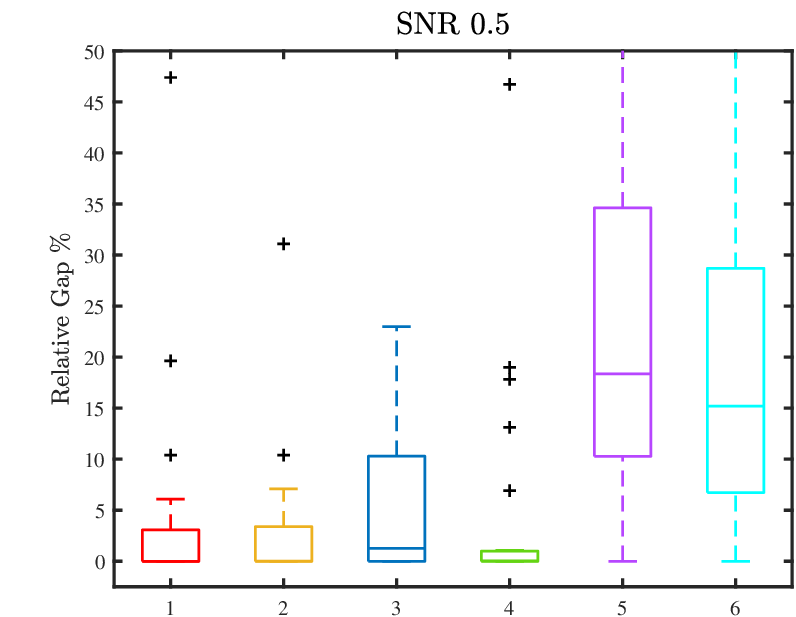}}

    \centering
    \subfloat{\includegraphics[width=0.32\linewidth]{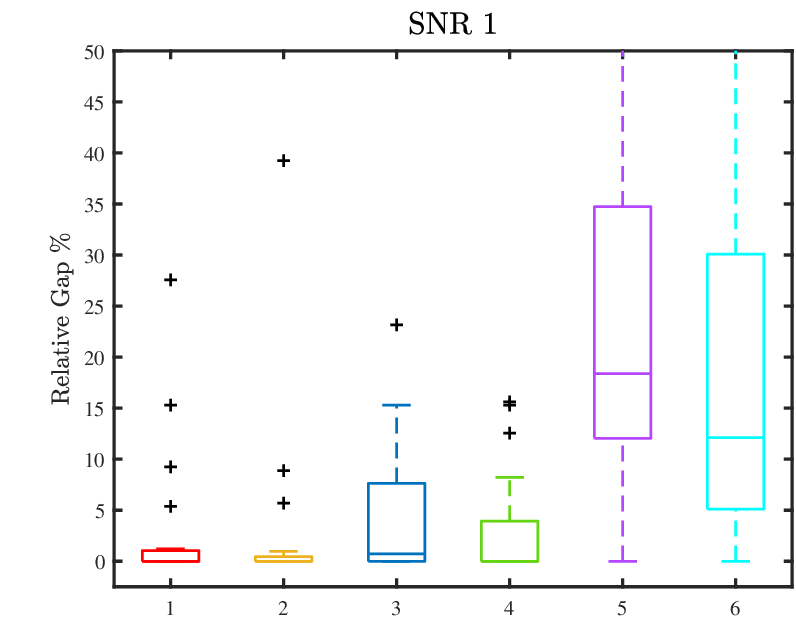}}
    \hfill
    \subfloat{\includegraphics[width=0.32\linewidth]{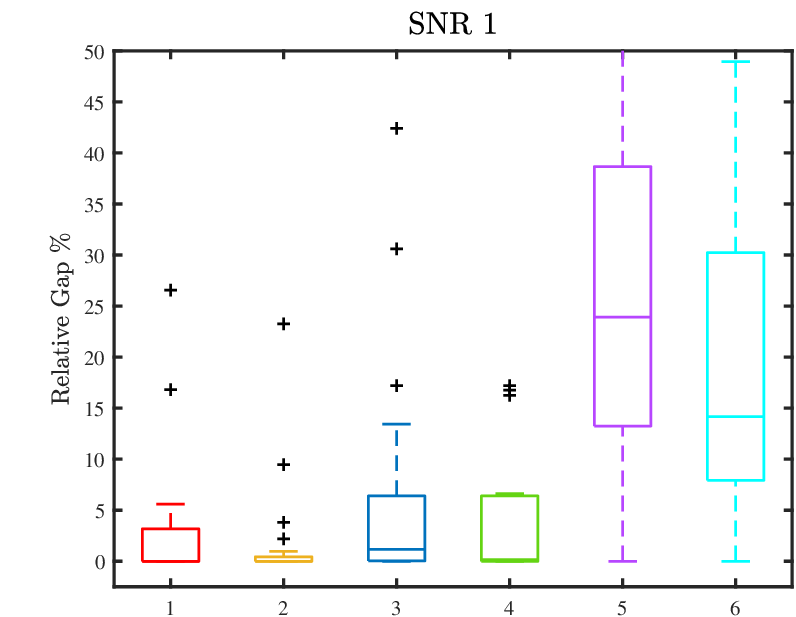}}
    \hfill
    \subfloat{\includegraphics[width=0.32\linewidth]{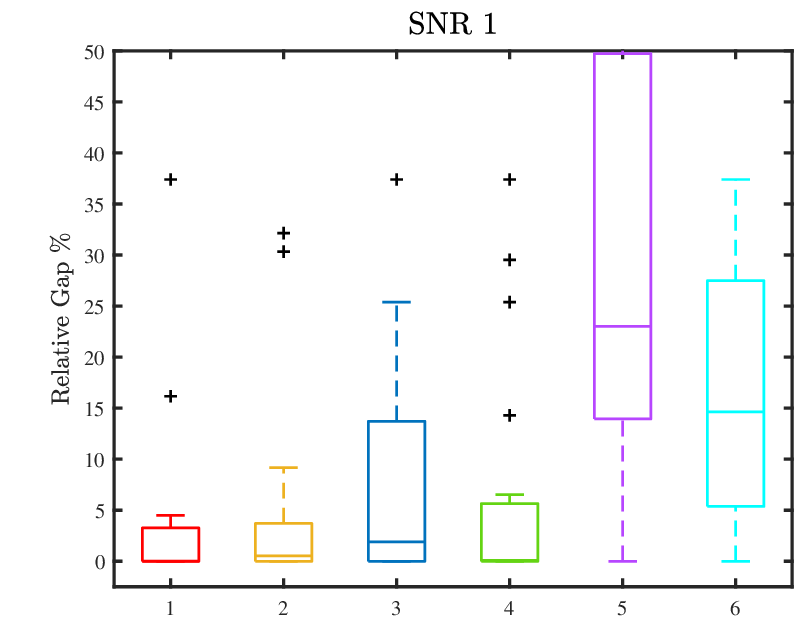}}

    \centering
    \subfloat{\includegraphics[width=0.32\linewidth]{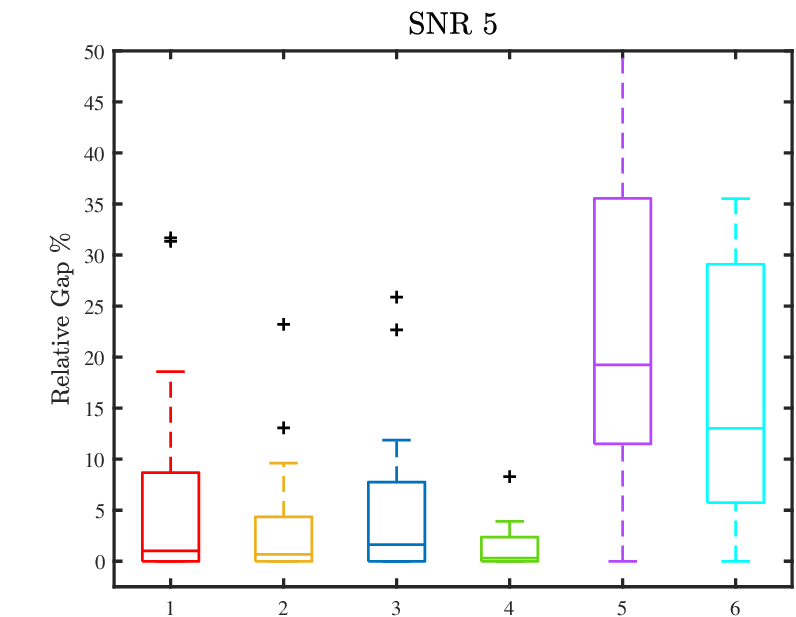}}
    \hfill
    \subfloat{\includegraphics[width=0.32\linewidth]{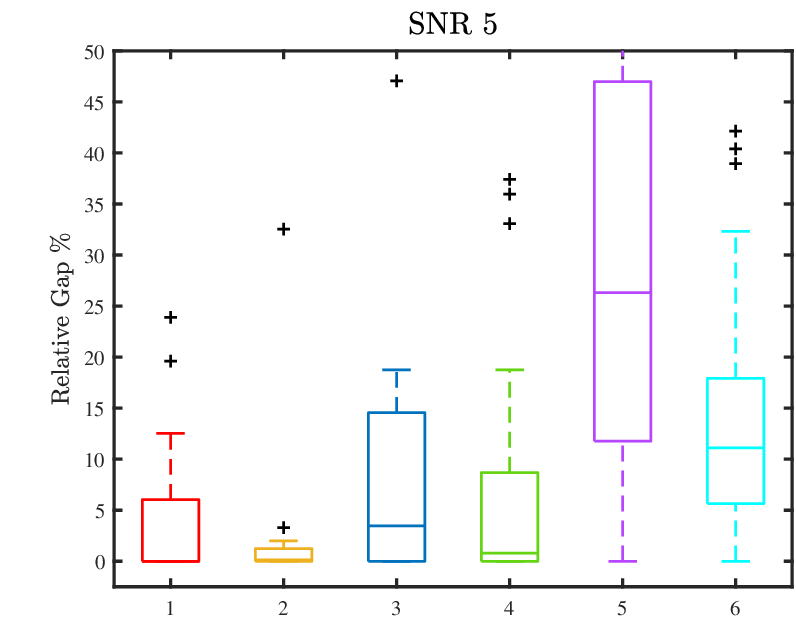}}
    \hfill
    \subfloat{\includegraphics[width=0.32\linewidth]{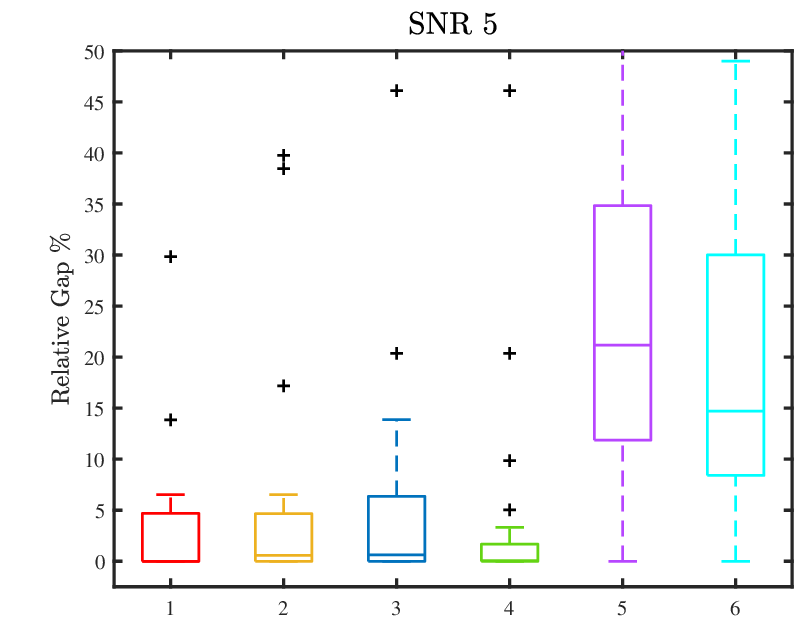}}
    \caption{Box plots of Relative Gap $\%$ for examples 1 (left), 2 (middle), and 3 (right) in UD case with exponential correlation in small, medium, and large dimension regimes with four SNR values and the two $k$ values where (1) SFS1; (2) SFS2; (3) FS; (4) SFFS; (5) GA; (6) DFOn}
    \label{fig:boxplotsUdEcEgsRelGapSfsFsSffsGaDfowrtSNR}
\end{figure}
\paragraph{CPU time performance}
Figures \ref{fig:perprofEgs12ROdEcCpuSfsFsSffsGaDfo} and \ref{fig:perprofEgs12RUdEcCpuSfsFsSffsGaDfo} show performance profiles of CPU time for examples 1, 2, and 3 with exponential correlation in all the three dimension regimes for the OD and the UD case, respectively, for those examples for which all the algorithms terminated without reaching any hard stop. The performance of the algorithms in terms of CPU time mirrors that of constant correlation, but 
FS and SFFS become comparable to SFS1 and even outperform SFS1 in the large dimension OD case. 

Again, our numerical results did not reveal a perfect algorithm. Under these circumstances, SFS1 and SFS2 are considered good alternatives to suboptimal algorithms for solving \eqref{bsschp5}. However, in the OD case, SFS2 became slower with exponential correlation data compared to the data with constant correlation. 
\begin{figure}
    \subfloat{\includegraphics[width=0.32\textwidth]{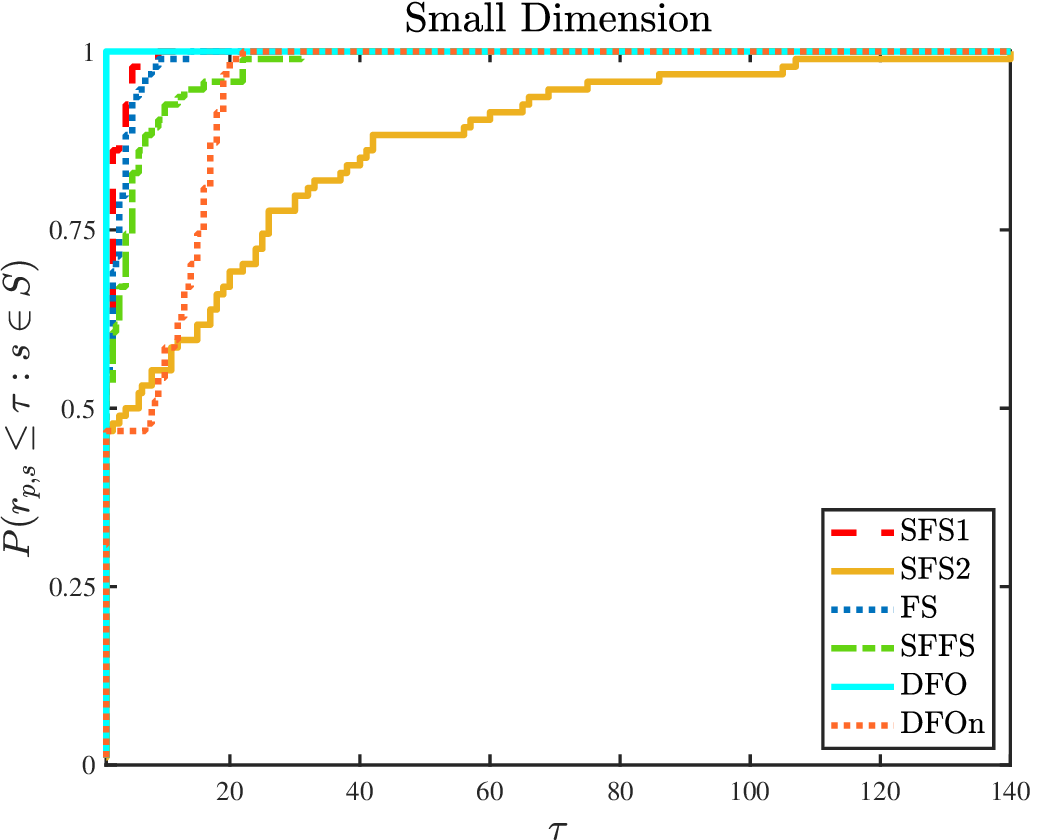}}
    \hfill
   \subfloat{\includegraphics[width=0.32\textwidth]{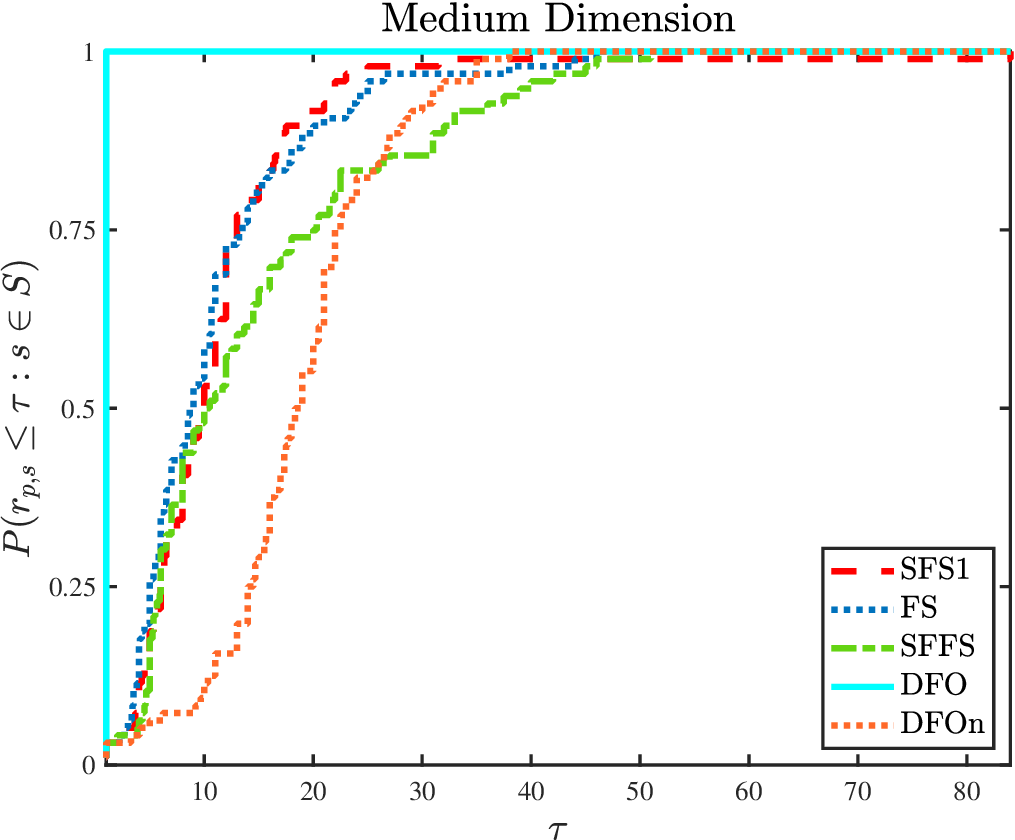}}
    \hfill
    \subfloat{\includegraphics[width=0.32\textwidth]{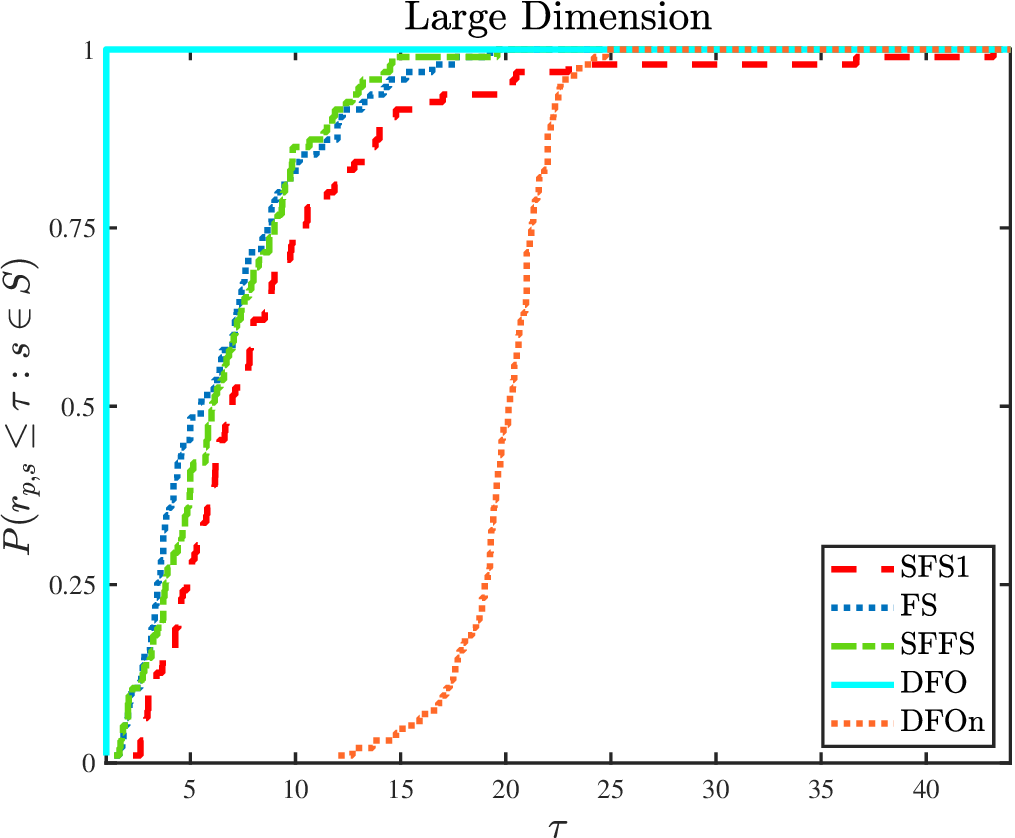}
     }
    \caption{Performance profiles of CPU time for examples 1, 2, and 3 combined in OD case with exponential correlation in small, medium, and large dimension regimes with four SNR values and the two $k$ values}
    \label{fig:perprofEgs12ROdEcCpuSfsFsSffsGaDfo}
\end{figure} 
%C:\Users\vikra\Desktop\testLLS\outLLSpackage05July\UbxEctype12R_6algSubOpt\combined\UDegs  for 99b3e9c   and \outLLSpackage10Oct24\  for 9d3f
\begin{figure}
    \subfloat{\includegraphics[width=0.32\textwidth]{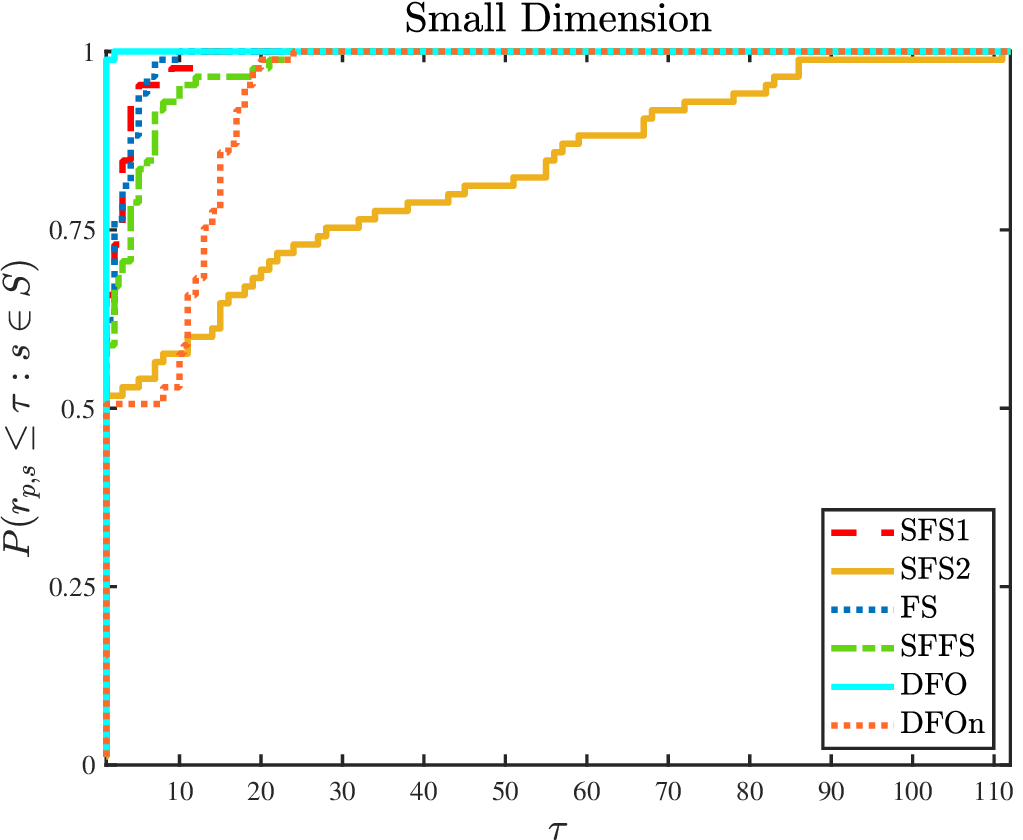}}
    \hfill
    \subfloat{\includegraphics[width=0.32\textwidth]{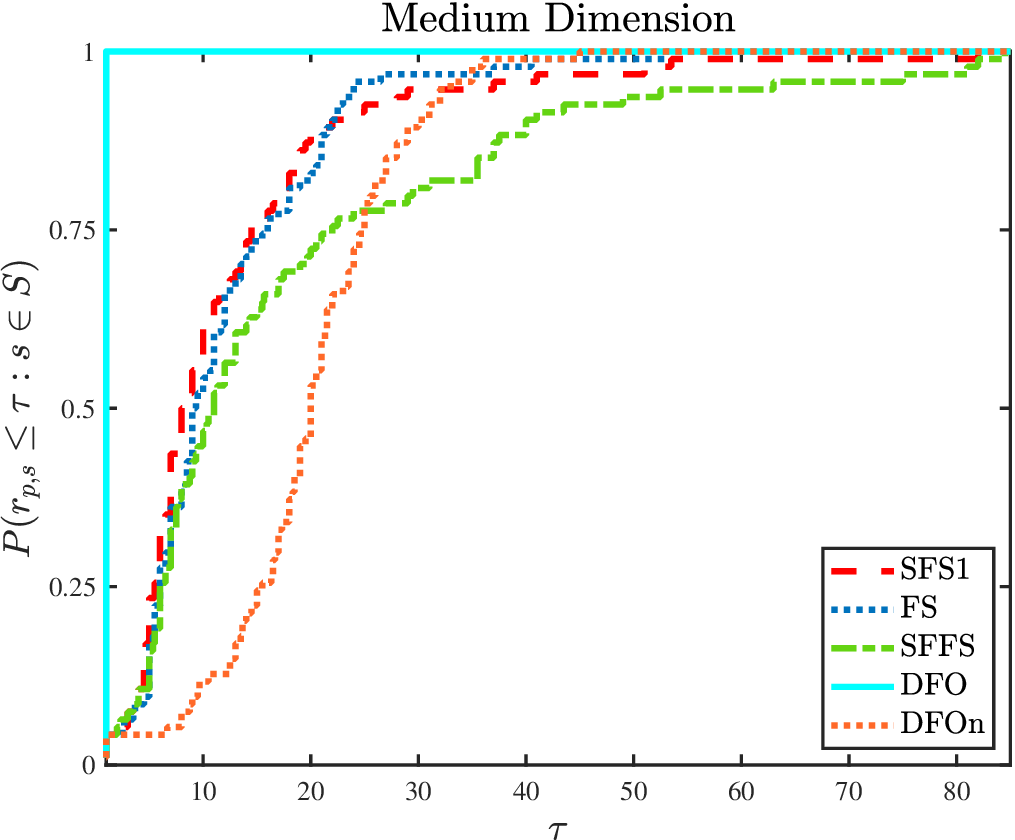}}
    \hfill
    \subfloat{\includegraphics[width=0.32\textwidth]{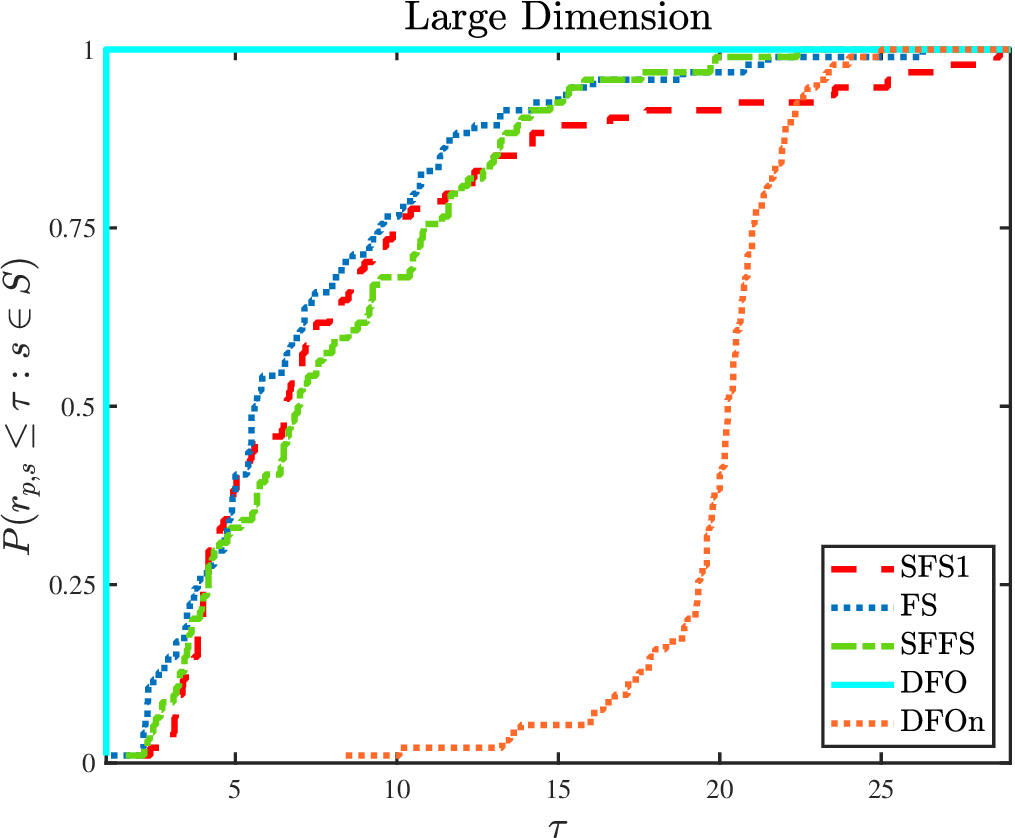}
     }
    \caption{Performance profiles of CPU time for examples 1, 2, and 3 combined in UD case with exponential correlation in small, medium, and large dimension regimes with four SNR values and the two $k$ values}
    \label{fig:perprofEgs12RUdEcCpuSfsFsSffsGaDfo}
\end{figure}
%==========================================================================
\subsection{Real data sets}
In this section, we present numerical test results for two real data sets. 
\subsubsection{Ozone data set} We use the Los Angeles Ozone data set downloaded from \url{https://hastie.su.domains/ElemStatLearn/}, which consists of 8 meteorological variables measured on 330 days in 1976 in Los Angeles. The response variable is the maximum hourly average ozone concentration. The first 8 variables are: \\
$X_{1}$ : 500 mb height \\
$X_{2}$ : wind speed \\
$X_{3}$ : humidity $\%$ \\
$X_{4}$ : surface temperature \\
$X_{5}$ : inversion height \\
$X_{6}$ : pressure gradient \\
$X_{7}$ : inversion temperature \\
$X_{8}$ : visibility \\
The remaining 36 variables have been generated as $X_{1}^{2}$, $X_{1}X_{2}$, $X_{2}^{2}$, $X_{1}X_{3}$, $X_{2}X_{3}$, $X_{3}^{2}$, $\cdots$, $X_{8}^{2}$. The design matrix $X$ and response $y$ have been normalized to have zero mean and unit $l_{2}$-norm. This results in a model with 44 predictors and 330 responses. This model falls into the middle of our small dimensional overdetermined examples.

Figure \ref{fig:boxPlotRelGapOzoneLeukemiaData} shows box plots of the Relative Gap $\%$ for the Ozone Data Set. Among the six algorithms tested, in terms of solution quality, SFS2 is the winner, followed by SFS1. Table \ref{tab:OzoneData_cputime} shows the CPU times for these algorithms. 
As the value of $k$ increases, the CPU time for SFFS can become significant. This is because SFFS may spend more time correcting the current model while searching for a model with a smaller RSS, at each iteration. In our testing, if SFFS has to stop by reaching the maximum CPU time limit, we use the last model with $k$ predictors found by SFFS as the final output. If SFFS is not able to build a model with $k$ predictors before reaching the maximum CPU limit, then the current model will be returned as the final output. Our test results show that for $k>26$, SFFS always returned a model with less than $k$ predictors in it. In this case, SFFS is clearly at a disadvantage compared to the other algorithms to find a model with the smallest RSS.
\begin{figure}[h!]
    \subfloat{\includegraphics[width=0.48\textwidth]{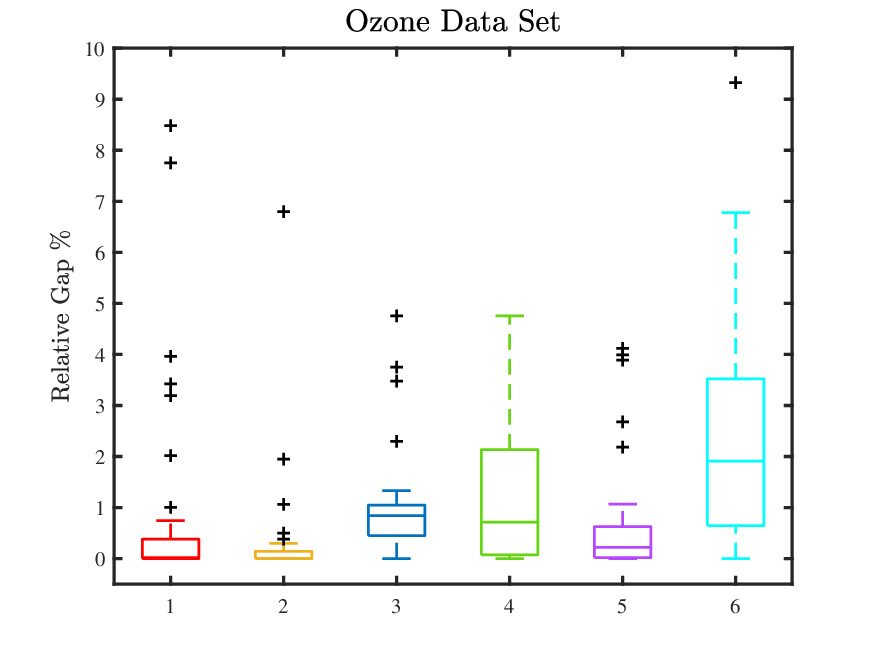}
    }
    \hfill
    \subfloat{\includegraphics[width=0.48\textwidth]{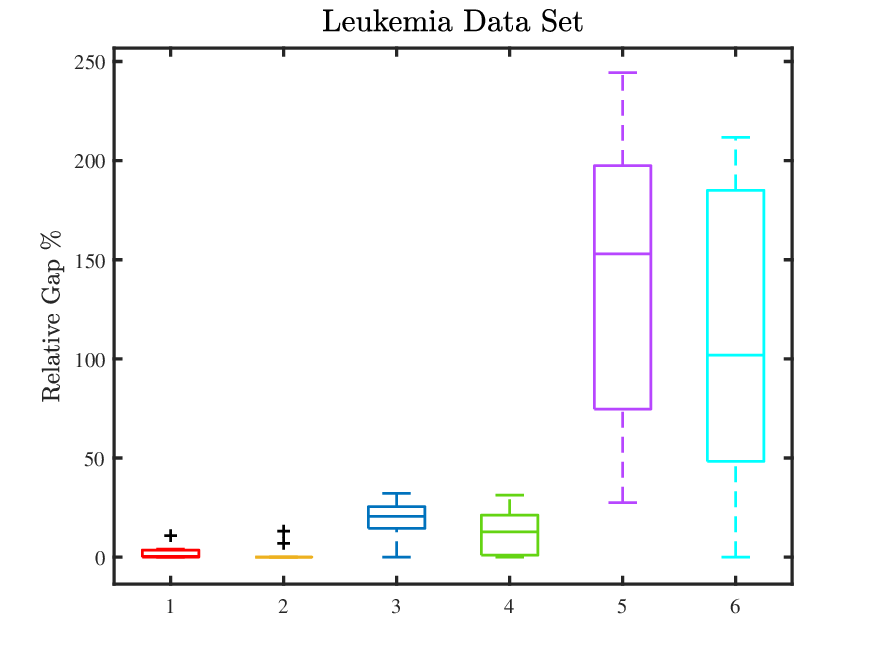}}
    \caption{Box plots of Relative Gap $\%$ for the Ozone Data Set and Leukemia Data Set where (1) SFS1; (2) SFS2; (3) FS; (4) SFFS; (5) GA; (6) DFOn}
    \label{fig:boxPlotRelGapOzoneLeukemiaData}
\end{figure}
\begin{table}[h!]
    \centering
    \caption{CPU time taken in seconds for each $k$ value by an algorithm for Ozone Data Set}
    \label{tab:OzoneData_cputime}
    %\begin{tabular}{|c|p{1cm}|p{1cm}|p{1cm}|p{1cm}|p{1cm}|p{1cm}|p{1cm}|}
    \begin{tabular}{llllllll}
    \hline\noalign{\smallskip}
     $k$  &  SFS1    &  SFS2    &  FS   &  SFFS &  GA & DFO & DFOn\\
    \noalign{\smallskip}\hline\noalign{\smallskip}
43 & 0.57 & 0.49 & 0.95 & $>$600 & 0.42 & 0.26 & 3.21 \\  
42 & 0.09 & 1.38 & 0.95 & $>$600 & 0.19 & 0.10 & 1.86  \\
41 & 0.16 & 4.55 & 0.24 & $>$600 & 0.21 & 0.10 & 1.58 \\ 
40 & 0.21 & 2.19 & 0.23 & $>$600 & 0.12 & 0.08 & 1.43  \\
39 & 0.21 & 1.80 & 0.25 & $>$600 & 0.21 & 0.08 & 1.43 \\ 
38 & 0.24 & 3.11 & 0.22 & $>$600 & 0.21 & 0.07 & 1.30 \\ 
37 & 0.21 & 2.16 & 0.21 & $>$600 & 0.25 & 0.06 & 1.26 \\ 
36 & 0.50 & 2.50 & 0.38 & $>$600 & 0.09 & 0.06 & 1.06 \\ 
35 & 0.77 & 1.68 & 0.20 & $>$600 & 1.24 & 0.06 & 1.08 \\ 
34 & 0.86 & 2.05 & 0.19 & $>$600 & 0.71 & 0.06 & 1.15 \\ 
33 & 0.34 & 2.99 & 0.19 & $>$600 & 0.55 & 0.15 & 1.65 \\ 
32 & 0.37 & 2.86 & 0.17 & $>$600 & 1.23 & 0.11 & 1.16 \\ 
31 & 0.30 & 2.10 & 0.17 & $>$600 & 1.38 & 0.04 & 0.81 \\ 
30 & 0.37 & 1.64 & 0.16 & $>$600 & 0.75 & 0.04 & 0.80 \\ 
29 & 0.36 & 1.46 & 0.16 & $>$600 & 0.57 & 0.03 & 0.63  \\
28 & 0.46 & 2.87 & 0.15 & $>$600 & 0.91 & 0.04 & 0.65 \\ 
27 & 0.22 & 1.38 & 0.15 & $>$600 & 0.86 & 0.03 & 0.64  \\
26 & 0.26 & 1.49 & 0.13 & $>$600 & 5.64 & 0.02 & 0.51 \\ 
25 & 0.23 & 1.18 & 0.13 & 0.50 & 4.78 & 0.02 & 0.46 \\ 
24 & 0.19 & 1.10 & 0.11 & 0.48 & 0.98 & 0.06 & 0.57  \\
23 & 0.24 & 1.15 & 0.11 & 0.47 & 2.95 & 0.06 & 0.61 \\ 
22 & 0.17 & 0.94 & 0.11 & 0.43 & 0.62 & 0.02 & 0.35 \\ 
21 & 0.41 & 1.30 & 0.36 & 1.16 & 0.66 & 0.02 & 0.52 \\ 
20 & 0.26 & 0.77 & 0.46 & 1.12 & 0.38 & 0.05 & 0.46 \\ 
19 & 0.13 & 0.62 & 0.39 & 0.40 & 0.69 & 0.01 & 0.41 \\ 
18 & 0.08 & 0.56 & 0.16 & 0.35 & 0.34 & 0.02 & 0.24 \\ 
17 & 0.07 & 0.63 & 0.08 & 0.28 & 0.31 & 0.01 & 0.22 \\ 
16 & 0.07 & 0.40 & 0.06 & 0.27 & 1.05 & 0.01 & 0.21 \\ 
15 & 0.07 & 0.39 & 0.06 & 0.12 & 0.62 & 0.02 & 0.22 \\ 
14 & 0.05 & 0.37 & 0.05 & 0.11 & 0.15 & 0.01 & 0.25 \\ 
13 & 0.05 & 0.63 & 0.05 & 0.11 & 0.22 & 0.01 & 0.20 \\ 
12 & 0.05 & 0.55 & 0.12 & 0.10 & 0.59 & 0.01 & 0.22  \\
11 & 0.07 & 0.51 & 0.04 & 0.09 & 0.29 & 0.00 & 0.14 \\ 
10 & 0.07 & 0.35 & 0.03 & 0.09 & 93.08 & 0.00 & 0.13 \\ 
9 & 0.05 & 0.33 & 0.03 & 0.07 & 0.55 & 0.01 & 0.13 \\ 
8 & 0.03 & 0.20 & 0.03 & 0.06 & 0.57 & 0.01 & 0.10  \\
7 & 0.04 & 0.32 & 0.03 & 0.05 & 0.03 & 0.01 & 0.13 \\ 
6 & 0.02 & 0.36 & 0.02 & 0.05 & 0.26 & 0.01 & 0.09 \\ 
5 & 0.02 & 0.29 & 0.02 & 0.02 & 0.04 & 0.01 & 0.09 \\ 
4 & 0.01 & 0.19 & 0.01 & 0.02 & 89.69 & 0.01 & 0.10 \\ 
3 & 0.02 & 0.68 & 0.01 & 0.00 & 0.01 & 0.00 & 0.08  \\
2 & 0.04 & 0.04 & 0.00 & 0.00 & 0.01 & 0.01 & 0.06 \\ 
1 & 0.05 & 0.04 & 0.01 & 0.01 & 0.02 & 0.01 & 0.12 \\ 
%    \hline
%    \hline
%    Total & 8.9  & 55.99 & 7.67   & $-$ & 9.87 & 1.1  \\ 
    \noalign{\smallskip}\hline
    \end{tabular}
\end{table}

\subsubsection{Leukemia data set} We considered the Leukemia Data Set downloaded from the supplementary material accompanying \cite{friedman2010regularization}. This data set consists of 3571 predictors and 72 observations with a binary response $y$, classifying a given sample into \textit{ALL} or \textit{AML} type leukemia. This model falls out of the dimensions we considered for the synthetic data sets in Table \ref{tab:sma_med_lar_eg_setup}. Thus, it is considered as an extreme case of large dimensional underdetermined examples. We changed the maximum CPU time limit to 1200 seconds for this testing.  

Figure \ref{fig:boxPlotRelGapOzoneLeukemiaData} and Table \ref{tab:LeukemiaData_cputime} show box plots of the Relative Gap $\%$ and list of CPU times for the Leukemia Data Set, respectively. SFS2 selects the models with the smallest RSS, followed by SFS1. In terms of CPU time, SFS1 is comparable to FS and SFFS and is significantly better than SFS2. In this testing, SFS1 emerges as the overall winner, providing good quality solutions within a reasonable CPU time.
\begin{table}[h!]
    \centering
    \caption{CPU time taken in seconds for each $k$ value by an algorithm for Leukemia Data Set}
    \label{tab:LeukemiaData_cputime}
    \begin{tabular}{llllllll}
    \hline\noalign{\smallskip}
     $k$  &  SFS1    &  SFS2    &  FS   &  SFFS &  GA & DFO & DFOn\\
    \noalign{\smallskip}\hline\noalign{\smallskip}
    10 & 6.85 & $>$1200 & 9.37 & 10.42 & $>$1200 & 1.00 & 19.60\\
    9 & 3.97 & $>$1200 & 4.23 & 5.18 & $>$1200 & 1.05 & 22.19\\
    8 & 3.13 & $>$1200 & 3.41 & 4.45 & $>$1200 & 0.97 & 21.18\\
    7 & 2.48 & $>$1200 & 2.90 & 2.90 & 228.02 & 0.92 & 18.60\\
    6 & 2.39 & $>$1200 & 2.40 & 2.42 & 225.36 & 0.94 & 18.30\\
    5 & 1.80 & $>$1200 & 1.92 & 1.91 & 225.85 & 0.92 & 18.18\\
    4 & 1.29 & 908.98 & 1.46 & 1.46 & 231.02 & 0.96 & 18.13\\
    3 & 0.81 & 858.31 & 1.02 & 1.03 & 231.92 & 0.88 & 18.13\\
    2 & 1.55 & 1.02 & 0.62 & 0.62 & 235.32 & 0.93 & 18.31\\
    1 & 1.99 & 1.64 & 0.28 & 0.28 & 230.86 & 0.95 & 18.73\\
    \noalign{\smallskip}\hline
    \end{tabular}
\end{table}
  
 \section{Discussion}
 In this work, we compared five suboptimal algorithms to solve the \eqref{bsschp5} problem. Two of them have two versions. The SFS and SFFS algorithms adapted to solve \eqref{bsschp5} show an advantage over the popular FS, GA, and DFO algorithms. SFS2 requires switching two predictors at each iteration, which increases the chance of finding a better model. However, finding the RSS for all the required new combinations of $k$ features at each iteration leads to a high computational cost, even for a moderate dimension problem. SFFS works by incorporating a strategy to correct a previously selected model with less than $k$ features while building the model with $k$ features. Due to this check, SFFS can take a lot of CPU time for a given example. SFS1 shows the desirable characteristics of finding good quality solutions for the \eqref{bsschp5} problem with less CPU time than SFS2 and SFFS algorithms. Thus, SFS1 is considered to be a competitive suboptimal algorithm for solving the best subset selection problem.
 \subsection*{Reproducibility}
 The test results in this paper can be reproduced by downloading the corresponding files from \hyperlink{https://github.com/vikrasingh/bss-suboptimal}{https://github.com/vikrasingh/bss-suboptimal}.
 
%\printbibliography
\bibliographystyle{plain} % We choose the "plain" reference style
\bibliography{main} 

\begin{thebibliography}{10}

\bibitem{bertsimasEtal:2015}
Dimitris Bertsimas, Angela King, and Rahul Mazumder.
\newblock Best subset selection via a modern optimization lens.
\newblock {\em The Annals of Statistics}, 44(2):813--852, 2015.

\bibitem{chang:1973}
Chieng~Yi Chang.
\newblock Dynamic programming as applied to feature subset selection in a pattern recognition system.
\newblock {\em IEEE Trans. Syst. Man Cybern.}, SMC-3(2):166--171, 1973.

\bibitem{dolan2002perprof}
Elizabeth~D. Dolan and Jorge~J. Mor{\'e}.
\newblock Benchmarking optimization software with performance profiles.
\newblock {\em Mathematical Programming}, 91(2):201--213, 2002.

\bibitem{friedman2010regularization}
Jerome Friedman, Trevor Hastie, and Rob Tibshirani.
\newblock Regularization paths for generalized linear models via coordinate descent.
\newblock {\em Journal of statistical software}, 33(1):1--22, 2010.

\bibitem{goldberg1989genetic}
David~E. Goldberg.
\newblock {\em Genetic Algorithms in Search, Optimization and Machine Learning}.
\newblock Addison-Wesley Professional, 1989.

\bibitem{marill1963effectiveness}
Thomas Marill and David~M. Green.
\newblock On the effectiveness of receptors in recognition systems.
\newblock {\em IEEE Trans. Info. Theory}, 9(1):11--17, 1963.

\bibitem{miller2002subset}
Alan Miller.
\newblock {\em Subset selection in regression}.
\newblock Chapman and Hall/CRC, 2 edition, 2002.

\bibitem{mucciardi1971comparison}
Anthony~N. Mucciardi and Earl~E. Gose.
\newblock A comparison of seven techniques for choosing subsets of pattern recognition properties.
\newblock {\em IEEE Trans. Comput.}, 100(9):1023--1031, 1971.

\bibitem{natarajan1995sparse}
Balas~K. Natarajan.
\newblock Sparse approximate solutions to linear systems.
\newblock {\em SIAM Journal on Computing}, 24(2):227--234, 1995.

\bibitem{pudil1994floating}
Pavel Pudil, Jana Novovi{\v{c}}ov{\'a}, and Josef Kittler.
\newblock Floating search methods in feature selection.
\newblock {\em Pattern Recognition Letters}, 15(11):1119--1125, 1994.

\bibitem{whitney1971fss}
A.~Wayne Whitney.
\newblock A direct method of nonparametric measurement selection.
\newblock {\em IEEE Trans. Comput.}, 100(9):1100--1103, 1971.

\bibitem{zhu2020polynomial}
Junxian Zhu, Canhong Wen, Jin Zhu, Heping Zhang, and Xueqin Wang.
\newblock A polynomial algorithm for best-subset selection problem.
\newblock {\em Proceedings of the National Academy of Sciences}, 117(52):33117--33123, 2020.

\end{thebibliography}

\end{document}